%% file: main.tex
\documentclass[sigconf]{acmart}
\usepackage{microtype}
\usepackage{graphicx}
\usepackage[caption=false,font=small]{subfig}
\usepackage{xcolor}
\usepackage{multirow}
\usepackage{makecell}
\usepackage{bm}
\usepackage{algorithm}
\usepackage[noend]{algorithmic}
\theoremstyle{plain}
\newtheorem{lemma}{Lemma}
\newtheorem{theorem}{Theorem}
\newtheorem{corollary}{Corollary}

\newcommand{\methodname}{FairMILE}
\theoremstyle{definition}
\newtheorem{definition}{Definition}
\newtheorem{problem}{Problem}

\AtBeginDocument{%
  \providecommand\BibTeX{{%
    \normalfont B\kern-0.5em{\scshape i\kern-0.25em b}\kern-0.8em\TeX}}}

\setcopyright{none}
\renewcommand\footnotetextcopyrightpermission[1]{}
\begin{document}

\title{\methodname{}: Towards an Efficient Framework for Fair Graph Representation Learning}

% \author{Anonymous authors}
% \affiliation{%
% %   \institution{The Ohio State University}
% %   \streetaddress{2015 Neil Ave}
% %   \city{Columbus}
% %   \state{Ohio}
%   \country{}
% %   \postcode{43210}
% }
% % \email{{he.1773,gurukar.1}@osu.edu, srini@cse.ohio-state.edu}

\author{Yuntian He}
\affiliation{%
  \institution{The Ohio State University}
  \streetaddress{2015 Neil Ave}
  \city{Columbus}
  \state{Ohio}
  \country{USA}
  \postcode{43210}
}
\email{he.1773@osu.edu}

\author{Saket Gurukar}
\authornote{Work done when the author was a graduate student at OSU.}
\affiliation{%
  \institution{Samsung Research America}
  \streetaddress{665 Clyde Ave}
  \city{Mountain View}
  \state{California}
  \country{USA}
  \postcode{94043}
  % \institution{The Ohio State University}
  % \streetaddress{2015 Neil Ave}
  % \city{Columbus}
  % \country{USA}
  % \postcode{43210}
}
\email{saket.gurukar@gmail.com}

\author{Srinivasan Parthasarathy}
\affiliation{%
  \institution{The Ohio State University}
  \streetaddress{2015 Neil Ave}
  \city{Columbus}
  \state{Ohio}
  \country{USA}
  \postcode{43210}
}
\email{srini@cse.ohio-state.edu}
% \affiliation{%

\renewcommand{\shortauthors}{He et al.}

\begin{abstract}
    Graph representation learning models have demonstrated great capability in many real-world applications. Nevertheless, prior research indicates that these models can learn biased representations leading to discriminatory outcomes. A few works have been proposed to mitigate the bias in graph representations. However, most existing works require exceptional time and computing resources for training and fine-tuning. To this end, we study the problem of efficient fair graph representation learning and propose a novel framework \methodname{}. \methodname{} is a multi-level paradigm that can efficiently learn graph representations while enforcing fairness and preserving utility. It can work in conjunction with any unsupervised embedding approach and accommodate various fairness constraints. Extensive experiments across different downstream tasks demonstrate that \methodname{} significantly outperforms state-of-the-art baselines in terms of running time while achieving a superior trade-off between fairness and utility.
\end{abstract}






\maketitle
%%%%%%%%%%%%%%%%%%%%%%%%%%%%%%%%%%%%%%%%%%%%%%%%%%%%%%%%%%%%%%%%%%%%%%%%%%%%%%%
%%%%%%%%%%%%%%%%%%%%%%%%%%%%%%%%%%%%%%%%%%%%%%%%%%%%%%%%%%%%%%%%%%%%%%%%%%%%%%%
% Main body

\section{Introduction}

\input{sections/introduction}

\section{Preliminaries}
\input{sections/statement}

\section{Related Work}
\input{sections/related}

\section{Methodology}
\input{sections/methodology}

\section{Experiments}
\input{sections/experiments.tex}

\section{Conclusion}
\input{sections/conclusion}

% \section{Introduction}
% \input{sections_eaamo/introduction}

% \section{Related Work}
% \input{sections_eaamo/related}

% \section{Preliminaries}
% \input{sections_eaamo/statement}

% \section{Methodology}
% \input{sections_eaamo/methodology}

% \section{Experiments}
% \input{sections_eaamo/experiments.tex}

% \section{Conclusion}
% \input{sections_eaamo/conclusion}

%%%%%%%%%%%%%%%%%%%%%%%%%%%%%%%%%%%%%%%%%%%%%%%%%%%%%%%%%%%%%%%%%%%%%%%%%%%%%%%
%%%%%%%%%%%%%%%%%%%%%%%%%%%%%%%%%%%%%%%%%%%%%%%%%%%%%%%%%%%%%%%%%%%%%%%%%%%%%%%

\begin{acks}
This material is supported by the National Science Foundation (NSF) under grants OAC-2018627, CCF-2028944, and CNS-2112471. Any opinions, findings, and conclusions in this material are those of the author(s) and may not reflect the views of the respective funding agency.
\end{acks}

\bibliographystyle{ACM-Reference-Format}
\bibliography{mybib}

% \newpage

% %%
% %% If your work has an appendix, this is the place to put it.
\appendix % \newpage \clearpage
\input{sections/techreport.tex}

\end{document}

%% file: sections/introduction.tex
A critical task in graph learning is to learn the hidden representations of the graph, also known as \textit{graph embedding}. The goal of graph embedding is to preserve both structure properties and node features in the graph. Such embeddings can be used to characterize individual users (e.g. Amazon and Netflix) and to promote new connections (e.g. LinkedIn).
Various  methods have been developed for this purpose~\cite{perozzi2014deepwalk, grover2016node2vec, qiu2018network, gurukar2022benchmarking}, including those based on graph neural networks (GNNs) \cite{kipf2017semi, hamilton2017inductive, xu2019how}. 
%Different from previous embedding methods, GNNs are usually designed to learn representations for a particular task. 
Such models have been effective in many real-world applications, such as crime forecasting \cite{jin2020addressing}, fraud detection \cite{wang2019semi}, and recommendation \cite{fan2019graph, gurukar2022multibisage}.

Given the high-stake decision-making scenarios that such models are typically deployed in, it is critical to ensure that the decisions made by these models are fair.  Prior studies \cite{dai2021say, rahman2019fairwalk, agarwal2021towards} reveal that graph representation learning models may inherit the bias from the underpinning graph data. A common source of bias is node features which may contain historical bias in sensitive attributes or other correlated attributes \cite{dwork2012fairness}. Another cause of bias is the \textit{homophily effect} - promoting links that may lead to increased segregation. Such bias can lead to a biased distribution in the embedding space \cite{dong2022fairness} and cause unfair treatment towards particular sensitive attributes such as gender and ethnicity \cite{dwork2012fairness}.
% \yh{For example, LinkedIn users connect more with their alumni, therefore people outside the community will gain less exposure in a who-to-follow recommendation task.}
There is a clear need to alleviate such bias, ideally without impacting the bottom line of model performance. % Move from paragraph 2 (commented due to redundancy); such platforms; (often quantified by accuracy, precision, recall, or some combination thereof)

%% -----

%% Merged

Recent efforts to address this problem seek to enhance fairness by adapting existing GNN models~\cite{dai2021say,agarwal2021towards,li2020dyadic,kang2020inform, current2022fairmod}. However, such adaptions often add to the models' complexity -- on large-scale graphs, these models either cannot finish execution in a reasonable amount of time or often result in an out-of-memory error. These concerns are amplified by recent articles that suggest that the training time for many AI models is simply becoming unsustainable ~\cite{lohn2022ai, strubell2019energy} with respect to both compute and emission costs. To tackle this issue, a naive solution is to apply scalability improvement techniques such as the multi-level framework ~\cite{karypis1998multilevel, liang2021mile, he2021distmile}. However, these solutions lack fairness considerations. \autoref{fig:preliminary} evaluates these approaches in terms of efficiency, fairness, and utility. The results demonstrate that: (1) Prior fairness-aware models are time-consuming for fair graph representation learning; (2) Scalable approaches like MILE \cite{liang2021mile} cannot enhance the fairness in embeddings. These observations highlight the challenges of balancing efficiency, fairness, and utility in the problem of fair graph representation learning.

% \yh{SP - Removed example as the graph size is too small to make the point - if both of you feel it should be kept - please uncomment. We could include it in conclusions.}
%For instance, FairGNN \cite{dai2021say} a representative approach for fair node classification, on the Credit graph with 30K nodes and 1.4 million edges, requires $53$ minutes to learn fair node embeddings. Later, we will show that our work can learn fairer and high-quality node embeddings than FairGNN in a mere 90 seconds on the same machine.

In addition to inefficiency, there are some other challenges with existing work. First, some works adapt existing unsupervised graph embedding approaches for fairness~\cite{rahman2019fairwalk, khajehnejad2021crosswalk}, but it is challenging to accommodate all such models. Second, many fair representation learning methods only consider a single, binary sensitive attribute, while real-world graphs usually have multiple multi-class sensitive attributes - limiting their applicability.

To address the above-mentioned issues, we present {\bf Fair} {\bf M}ult{\bf I} {\bf L}evel {\bf E}mbedding framework (\methodname{}). \methodname{} is a general framework for fair and efficient graph representation learning. It adopts a multi-level framework used by recent scalable embedding methods \cite{liang2021mile, akyildiz2020gosh, chen2018harp, he2021distmile, he2022webmile}. However, unlike other multi-level frameworks, our framework incorporates fairness as a first-class citizen. 
% in which the modules are tailored for fairness considerations. 
Our framework is method agnostic in that it can accommodate any unsupervised graph embedding method treating it as a black box. % and simultaneously improve both its scalability and fairness.
Moreover, unlike a majority of fair graph representation learning models, \methodname{} can learn fair embeddings with respect to multiple multi-class sensitive attributes simultaneously.
% \footnote{Later, we show that \methodname{} is competitive with CFGE \cite{bose2019compositional} in terms of quality -- a method that operates on multiple multi-class sensitive attributes --  in significantly less time ($46\times$ speedup). }
To summarize, our main contributions are:

\begin{figure}[t]
\vskip -0.1in
\centering
\includegraphics[width=\linewidth]{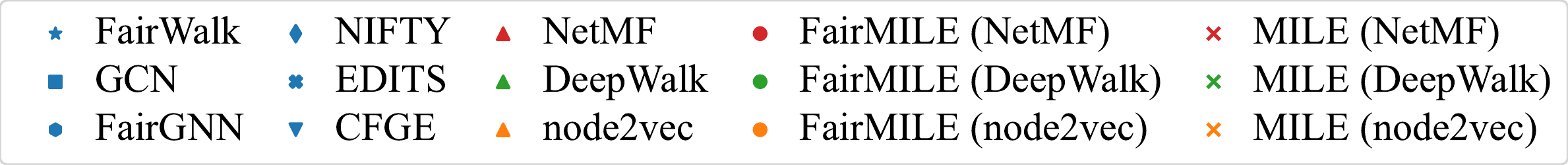}\\
\includegraphics[width=\linewidth]{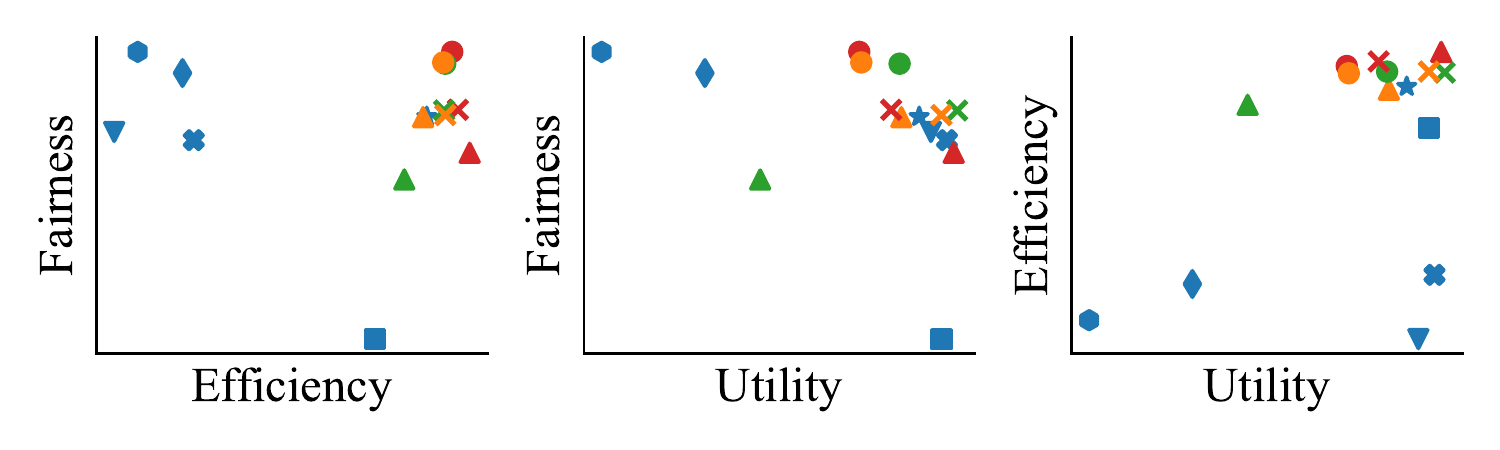}
\vskip -0.1in
\caption{Evaluation of fairness-aware methods, graph embedding methods, MILE, and this work (\methodname{}) in terms of efficiency, fairness, and utility. The top right corner of each plot corresponds to the best performance. This demonstrates that existing fair representation learning works are inefficient and scalable embedding methods are unable to enforce fairness, which highlights our novelty and contributions. } % Evaluation of GCN, fairness-aware GNN (FairGNN), scalable embedding (MILE) and this work in terms of running time in seconds, fairness ($\Delta_{DP}$), and utility (AUROC). This demonstrates that existing works on fair representation learning are inefficient and scalable embedding methods are unable to enforce fairness, which highlights the novelty and contribution of this study.
\label{fig:preliminary}
\Description[Evaluation of the tradeoff between efficiency, fairness, and utility.]{Evaluation of the tradeoff between efficiency, fairness, and utility. The top right corner of each plot corresponds to the best performance.}
\vskip -0.15in
\end{figure}

\begin{itemize}
    \item \textbf{Novelty}: To the best of our knowledge, this is
the first work that seeks to improve the efficiency issue present in fair graph representation learning. To that end, we develop a general-purpose framework called \methodname{}.
    % \item We study the problem of fair graph representation learning and develop a general-purpose framework called \methodname{}. \yh{To the best of our knowledge, it is the first work addressing the efficiency issue of fair graph representation learning.} % is unsupervised and task-agnostic, therefore it can be used in various downstream tasks such as link prediction and node classification. 
    % \item Our work improves the \textbf{scalability} of fair graph representation learning. Compared to prior works that need a long time to train and fine-tune the hyperparameters, \methodname{} is easy to use and learns fair embeddings efficiently. 
    \item \textbf{Model-agnostic}: \methodname{} can easily accommodate any unsupervised graph embedding methods and improve their fairness while preserving the utility.
    \item \textbf{Efficiency and versatility}: Compared with existing approaches, \methodname{} successfully improves the efficiency of fair graph representation learning. In addition, \methodname{} can achieve fairness towards \textit{multiple} and \textit{non-binary} sensitive attributes, which most prior works fail to consider.
% while preserving the utility. 
    % \item \textbf{Versatility}: \methodname{} is agnostic to downstream task and base model, which means it can easily accommodate any unsupervised embedding method for fair representation learning. Meanwhile, it can achieve fairness towards multiple and non-binary sensitive attributes.
    \item \textbf{Evaluation}: We demonstrate both the efficacy and efficiency of \methodname{} across both node classification and link prediction settings. Our results show that \methodname{} can improve efficiency by up to two orders of magnitude and fairness by several factors while realizing comparable accuracy to competitive strawman. % For instance when performing node classification on {\it Credit} - a popular dataset for such analysis - % Similarly, for link prediction on {\it PubMed}, we see {\bf over an order of magnitude of improvement in efficiency, and significant improvement in fairness, with comparable accuracy}.
\end{itemize}

%The rest of this paper is organized as follows. In Section 2, we survey the prior work in relevant fields including fairness in machine learning and scalable graph embedding. Section 3 introduces important notations and definitions used in this paper, and formulates the problem of fair representation learning. In Section 4, we propose our \methodname{} framework and explicitly explain how it incorporates fairness and scalability with graph embedding. Section 5 evaluates the performance of \methodname{} and other baselines through their utility and fairness in real-world graph mining tasks. In Section 6, we conclude this paper and give several directions for future work.

%% file: sections/statement.tex
% In this section, we introduce the necessary notations and fairness definitions used throughout the paper. Then we formally define the problem of fair graph representation learning. 

\subsection{Notations}
Let $\mathcal{G} = (\mathcal{V}, \mathcal{E})$ be an undirected graph, where $\mathcal{V}$ is the set of nodes, and $\mathcal{E} \subseteq \mathcal{V} \times \mathcal{V}$ is the set of edges. Let $\bm{A}$ be the graph adjacency matrix, where $\bm{A}_{u, v}$ denotes the weight of edge $(u, v)$. $\bm{A}_{u, v} = 0$ means $u, v$ are not connected. $\delta(u)$ denotes the degree of node $u$. The graph also contains a set of sensitive attributes $\mathcal{F}$ (e.g., gender and race). Each attribute may have a binary or multi-class value associated with a particular demographic group.

\subsection{Fairness Metrics}
In this paper, we focus on group fairness and use two metrics for evaluation. These metrics have been widely adopted in prior works~\cite{agarwal2021towards,dong2022edits,dwork2012fairness,dai2021say,rahman2019fairwalk}. Without loss of generality, we first introduce them in a binary prediction scenario with a binary sensitive attribute, and then we extend them to a general multi-class case.

\begin{definition} \textbf{\textit{Demographic Parity}} (also known as \textit{statistical parity}) \cite{dwork2012fairness} requires that each demographic group should receive an \textit{advantaged} outcome (i.e., $\hat{Y}>0$) at the same rate, which is formulated as $P(\hat{Y} = 1 | S = 0) = P(\hat{Y} = 1 | S = 1)$, where $\hat{Y}$ is the predicted label and $S \in \mathcal{F}$ is a binary sensitive attribute. To quantify how demographic parity is achieved, prior works \cite{dai2021say, lahoti13operationalizing, agarwal2021towards} define $\Delta_{DP, \mathrm{binary}}$ as
\begin{equation}
    \Delta_{DP, \mathrm{binary}} = |P(\hat{Y} = 1 | S = 0) - P(\hat{Y} = 1 | S = 1)| \nonumber
\end{equation}

To extend demographic parity for multi-class sensitive attributes, Rahman et al. \cite{rahman2019fairwalk} measure the variance of positive rates among all groups. Here, we further extend it to a multi-class predicted label scenario by averaging the standard deviations (denoted as $\sigma$) across all advantaged classes. The formulation is given as
\begin{equation}
    \Delta_{DP} = \frac{1}{|\mathcal{Y}^{+}|} \sum_{y \in \mathcal{Y}^{+}} \sigma \left ( \{ P(\hat{Y} = y | S = s ): \forall s \} \right ) \nonumber
\end{equation}
where $\mathcal{Y}^{+}$ denotes a set of advantaged classes.
\end{definition}

\begin{definition}
\textbf{\textit{Equality of Opportunity}} requires that each demographic group has an identical probability to receive a specific advantaged outcome for its members with this advantaged ground-truth label. In binary classification tasks with respect to a binary sensitive attribute, existing works \cite{dai2021say, agarwal2021towards, dong2022edits} compute the difference of true positive rates across two groups to measure the equality of opportunity, which is formulated as
\begin{equation}
\begin{split}
    \Delta_{EO, \mathrm{binary}} = ~ & |P(\hat{Y} = 1 | Y = 1, S = 0) - P(\hat{Y} = 1 | Y = 1, S = 1)| \nonumber
\end{split}
\end{equation}
where $Y$ is the ground truth label. Similarly,
% Similar to the extended formulation of demographic parity, 
we define a new metric $\Delta_{EO}$ to extend equality of opportunity to a general scenario with multi-class labels and attributes:
\begin{equation}
    \Delta_{EO} = \frac{1}{|\mathcal{Y}^{+}|} \sum_{y \in \mathcal{Y}^{+}} \sigma \left ( \{ P(\hat{Y} = y | Y = y, S = s): \forall s \} \right ) \nonumber
\end{equation}
\end{definition}

Note that the previously adopted metrics for binary predicted labels and binary sensitive attributes are a special case of our proposed measures, i.e., $\Delta_{DP} = \frac{1}{2}\Delta_{DP, \mathrm{binary}}$ and $\Delta_{EO} = \frac{1}{2}\Delta_{EO, \mathrm{binary}}$.

\subsection{Problem Statement}
% With the notations and fairness metrics introduced above, we formulate the problem of fair graph representation learning as:

\begin{problem} Given a graph $\mathcal{G} = (\mathcal{V}, \mathcal{E})$, the embedding dimensionality $d$, and a set of sensitive attributes $\mathcal{F}$, the problem of \textbf{Fair Graph Representation Learning} aims to learn a fair embedding model $f : \mathcal{V} \rightarrow \mathbb{R}^d$ with less inherent bias towards attributes in $\mathcal{F}$ where the present bias is measured with $\Delta_{DP}$ and $\Delta_{EO}$.
\end{problem}

% \yh{To evaluate the fairness of the learned representations, we can use the embeddings to train a machine learning model for a downstream task. Then we can observe the bias within the predicted outcomes in terms of the metrics proposed above.}

% be able to make quality predictions in particular downstream tasks (e.g., node classification and link prediction) and satisfy the desired fairness criteria (e.g., demographic parity and equality of opportunity).

%% file: sections/related.tex
% In this section, we review the existing works on fairness in machine learning with a particular focus in the field of fair graph representation learning. We also survey the scalable graph embedding methods.

\subsection{Fairness in Machine Learning}
% Machine learning techniques have been ubiquitously used in everyday scenarios such as recommendation~\cite{wang2018billion}, candidate screening~\cite{cohen2019efficient}, and fraud detection~\cite{wang2019semi}.
Since machine learning techniques are deployed to make decisions that have societal or ethical implications~\cite{wang2018billion, cohen2019efficient, wang2019semi}, serious concerns over their fairness are raised. There have been various definitions of fairness in machine learning. In this paper, we focus on the most popular definition \textit{group fairness}~\cite{dwork2012fairness},
% is the most widely used one, 
which requires that an algorithm should treat each demographic group equally. The groups are associated with a single or multiple \textit{sensitive attributes}, such as gender and race. There are also other definitions of fairness including \textit{individual fairness}~\cite{dwork2012fairness} and \textit{counterfactual fairness}~\cite{kusner2017counterfactual}.

% Other definitions include \textit{individual fairness}, which is satisfied when similar individuals receive similar outcomes \cite{dwork2012fairness}, and \textit{counterfactual fairness}, which considers an algorithm is fair if it does not change the decision when an individual's sensitive attribute value changes \cite{kusner2017counterfactual}. 

Unfair outcomes are mostly caused by data bias and algorithmic bias~\cite{mehrabi2021survey}. There exists a wide range of biases in data. For example, features like home address can be associated with specific races and lead to unfair decisions indirectly~\cite{dong2022fairness}. The design of machine learning algorithms may also unintentionally amplify the bias in data. To address this concern, several fair machine learning algorithms have been proposed in recent years~\cite{singh2018fairness,asudeh2019designing,fu2020fairness,zafar2017fairness,bolukbasi2016man}. A comprehensive survey on fair machine learning is given in~\cite{mehrabi2021survey}.

\begin{table}[t]
\caption{Summary of fair graph representation learning methods.} %  (MA denotes method agnostic, MS denotes multiple sensitive attributes, and NB denotes non-binary sensitive attributes)
\label{table:summary_fgrl}
\vspace{-0.1in}  % at least 0.1 inches
\begin{center}
\begin{small}
% \begin{sc}
\resizebox{\linewidth}{!}{
\begin{tabular}{l|ccc}
\toprule
% \makecell[c]{\multirow{3}{*}{Method}} & \multirow{3}{*}{\minitab[c]{Task\\Agnostic}} & \multirow{3}{*}{\minitab[c]{Multiple\\Attr.}} & \multirow{3}{*}{ \minitab[c]{Multi-\\class\\Attr.}} \\
%  \\
%  \\
% Method & MA & MS & NB \\
Method & Method-agnostic & Multiple Sensitive & Non-binary \\ % Single column
& & Attributes & Attributes \\
\midrule
FairGNN \cite{dai2021say} & $\times$ & $\times$ & $\times$ \\
NIFTY \cite{agarwal2021towards} & $\times$ & $\times$ & $\times$ \\
FairAdj \cite{li2020dyadic} & $\times$ & $\times$ & $\surd$ \\
FairWalk \cite{rahman2019fairwalk} & $\times$ & $\times$ & $\surd$ \\
% FairDrop \cite{spinelli2021fairdrop} & $\surd$ & $\times$ & $\surd$ \\
CFGE~\cite{bose2019compositional} & $\times$ & $\surd$ & $\surd$ \\
EDITS \cite{dong2022edits} & $\surd$ & $\times$ & $\surd$ \\
\methodname{} (This work) & $\surd$ & $\surd$ & $\surd$\\
\bottomrule
\end{tabular}
}
% \end{sc}
\end{small}
\end{center}
\vspace{-0.1in}
\end{table}

\subsection{Fair Graph Representation Learning}
%For similar reasons, most graph learning algorithms lack fairness considerations.
% It is even more challenging to incorporate graph learning with fairness because graph data is non-i.i.d.~\cite{dong2022fairness,dai2021say,wu2021learning}. 
In the graph context, models trained to realize representations accounting for the connectivity and topology inherent to the network (e.g. homophily bias) can lead to biased representations.
Downstream tasks that operate on such representations can lead to unfair recommendations~\cite{linkedinMITreview}, and even biased and unjust outcomes~\cite{maneriker2023online}.  However, the incorporation of fairness with graph-based learning is challenging because of the non-i.i.d nature of the data and the homophily effect of graph data~\cite{dong2022fairness,dai2021say,jalali2020information}. 

Recently, several methods have been proposed to learn fair graph representations.
FairGNN~\cite{dai2021say} leverages adversarial learning to train fair GNNs for node classification. NIFTY~\cite{agarwal2021towards} adds a fairness loss to the GNN objective as regularization. FairAdj~\cite{li2020dyadic} accommodates the VGAE~\cite{kipf2016variational} model for fair link prediction. For task-agnostic embedding, FairWalk~\cite{rahman2019fairwalk} learns fair embeddings by adapting an embedding algorithm node2vec~\cite{grover2016node2vec}. Specifically, it modifies the random walk process and adjusts the probability of selecting nodes in each sensitive group for fairness. To consider multiple sensitive attributes, CFGE~\cite{bose2019compositional} employs a set of adversaries with the encoder for compositional fairness constraints. Unlike these approaches, EDITS~\cite{dong2022edits} is a pre-processing solution that reduces the bias in graph structure and node attributes, then trains vanilla GNNs on the debiased graph. However, since most methods are GNN-based, they require exceptional time for training and fine-tuning.

In addition to inefficiency, there are three other major drawbacks of existing works. First, some of them are not \textit{method-agnostic} which means they require non-trivial modifications to the base model for adaption. Second, most existing works are unable to incorporate fairness constraints towards \textit{multiple} sensitive attributes. Third, some methods cannot handle \textit{non-binary} sensitive attributes. \autoref{table:summary_fgrl} summarizes these approaches together with our proposed work \methodname{}. Compared with these approaches, our work can simultaneously (1) accommodate the base model easily while (2) achieve fairness towards multiple non-binary sensitive attributes. Most importantly, we will demonstrate that (3) \methodname{} significantly outperforms these baselines in terms of efficiency. % by orders of magnitude}.

\begin{figure*}[!t]
  \centering
  \subfloat[An illustrative pipeline of \methodname{}]{\includegraphics[width=0.4\linewidth]{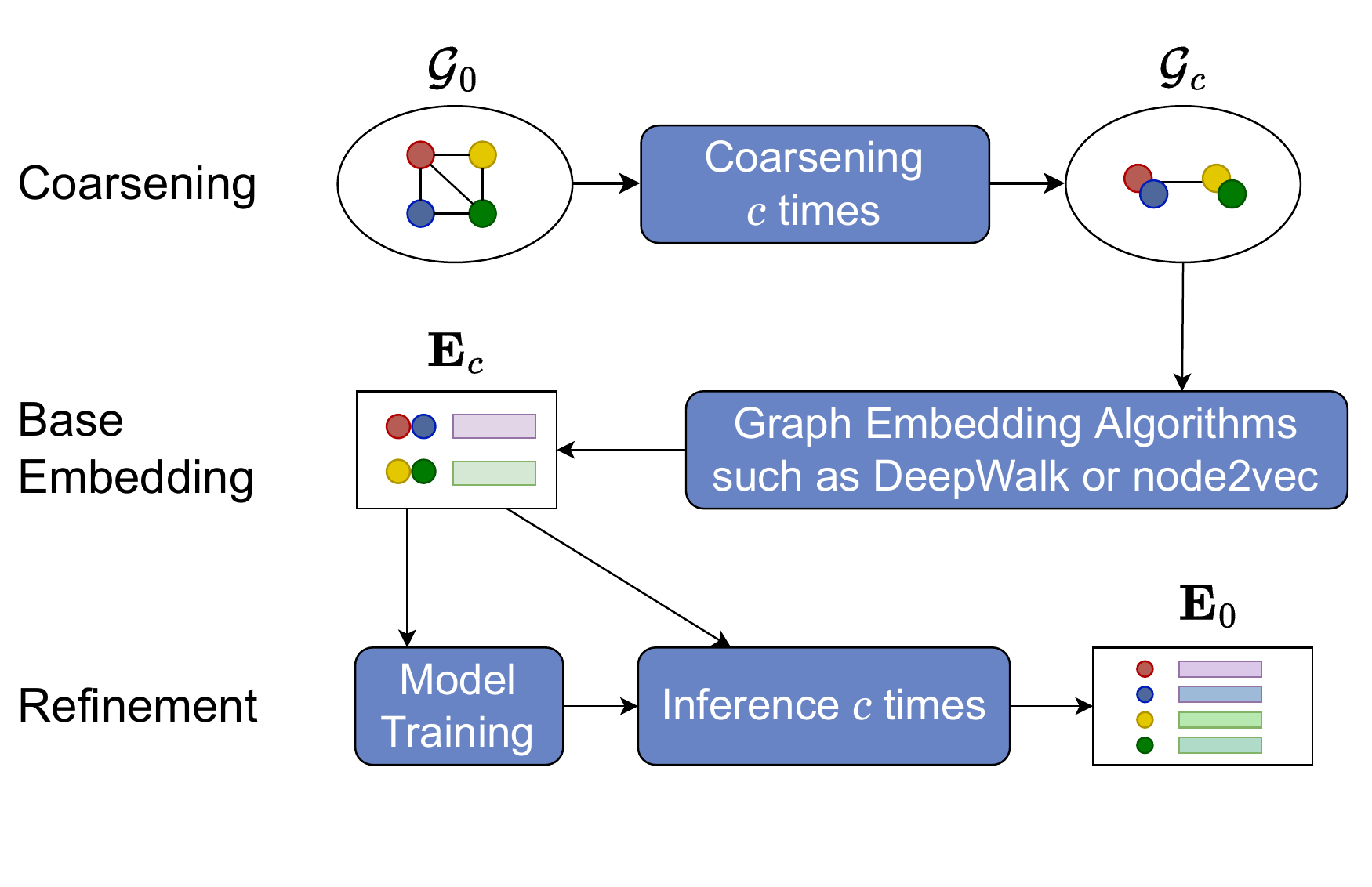}
    \label{fig:overview}}
  \hspace{.05\linewidth}
  \subfloat[Architecture of the refinement model.]{\includegraphics[width=0.4\linewidth]{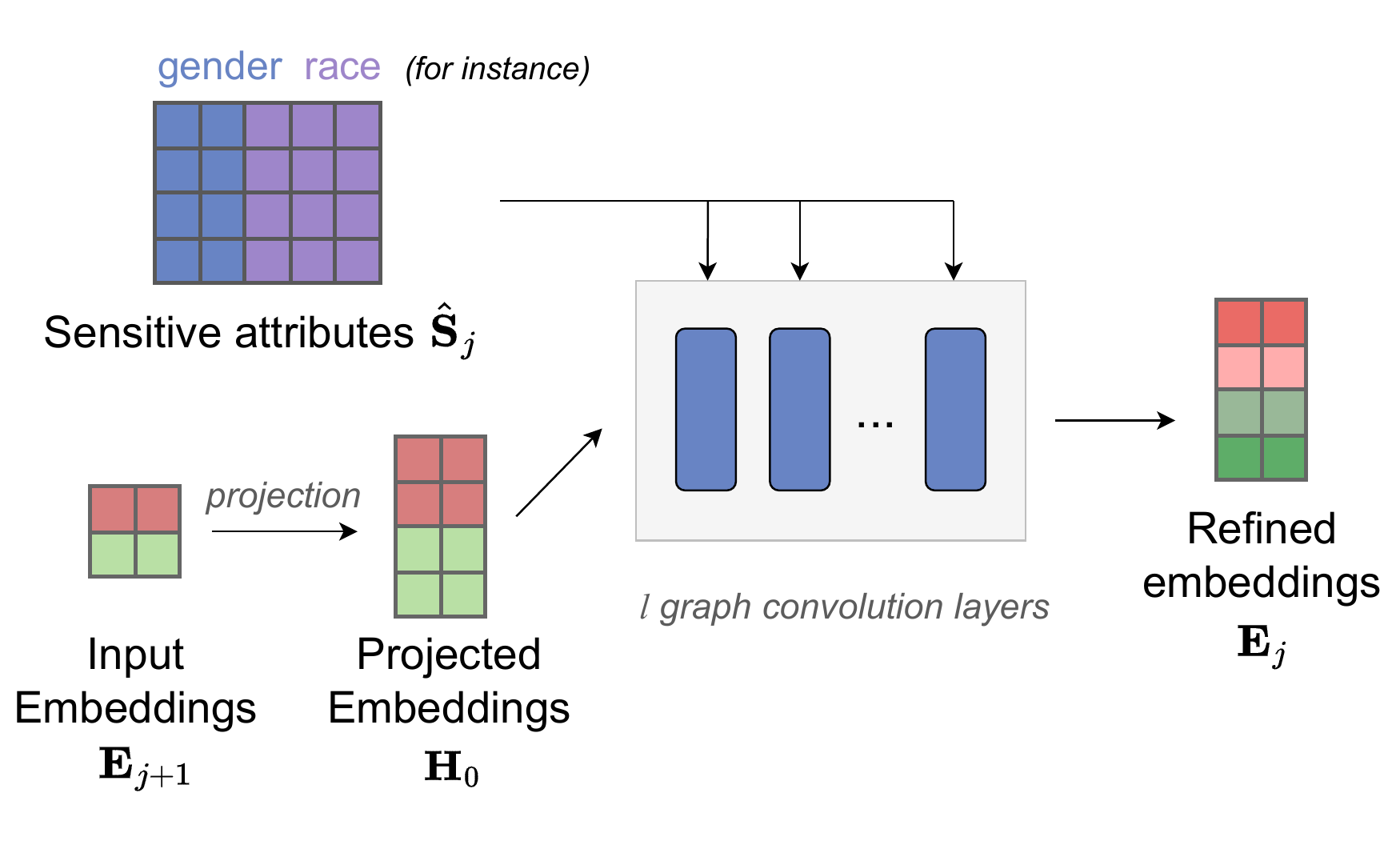}
    \label{fig:refinement}}
  \vspace{-.1in}
  \caption{Overview of \methodname{}}
  \label{fig:overview_all}
  \Description[The pipeline of FairMILE]{FairMILE consists of coarsening, base embedding, and refinement.}
  \vspace{-.1in}
\end{figure*}

\subsection{Scalable Graph Embedding}
% Graph embedding aims to learn low-dimensional representations on graph data, which can capture the underlying topological properties and feature information of a given graph. 
Many methods for graph embedding have been proposed in recent years, including NetMF \cite{qiu2018network}, DeepWalk \cite{perozzi2014deepwalk}, and node2vec \cite{grover2016node2vec}. Despite their excellent performance on various machine learning tasks, their lack of scalability prohibits them from processing large datasets. Recent research addressed the scalability issue of graph embedding using different methodologies. Some studies leverage high-performance computing techniques \cite{zhu2019graphvite, lerer2019pytorch, qiu2021lightne}. Another group of studies adopts the multi-level framework for better scalability. This framework is widely used for various graph problems \cite{karypis1998multilevel, liang2021mile, chen2018harp} and the essential idea is to solve the problem from a smaller coarsened graph. However, none of these methods are fairness-aware. We are the first study that considers fairness based on this framework. % (Longer version) The essential idea of the multi-level framework is to first coarsen the input graph into a smaller one, then learn the embeddings on the coarsened graph and refine them. These methods learn representations efficiently with little compromise on result quality.

%% file: sections/methodology.tex
% Figure \ref{fig:overview} illustrates the pipeline of \methodname{}, which consists of three modules: graph coarsening, base embedding, and refinement. It is built on a multi-level framework, which is a general approach for efficiency improvement in various graph problems, such as partitioning \cite{karypis1998multilevel} and embedding \cite{liang2021mile, chen2018harp}. The idea is to first coarsen the original graph into a smaller one, then learn the embeddings of the coarsened graph, and eventually refine it into a final solution. However, our preliminary experiment in \autoref{fig:preliminary} reveals that naively using this framework cannot improve fairness. Next, we introduce the functionality of each module and explain how \methodname{} enforces fairness in the graph representations while improving efficiency and versatility.

We propose a fairness-aware graph embedding framework \methodname{} (shown in Figure \ref{fig:overview}), which consists of three modules: graph coarsening, base embedding, and refinement. The idea is to first coarsen the original graph into a smaller one, then learn the embeddings of the coarsened graph using the base model, and eventually refine them into the embeddings of the original graph.

% The key difference between this work and other embedding frameworks \cite{liang2021mile, chen2018harp} of the similar architecture lies in the \textbf{novel algorithm designs specifically for fairness concerns}. Note that our preliminary experiment in \autoref{fig:preliminary} reveals that those prior works cannot improve fairness.

There exist several definitions of fairness~\cite{dwork2012fairness,kusner2017counterfactual} - all with the shared principle that all subgroups should receive the positive outcome at the same level in a given measure (e.g., positive rate for $\Delta_{DP}$, true positive rate for $\Delta_{EO}$). Since this paper studies unsupervised representation learning, \textbf{our key idea} is to minimize the variance among the representations of different groups. Intuitively, the downstream models trained with such representations (or embeddings) will lead to a decrease in bias. Next, we introduce the functionality of each module and explain how \methodname{} enforces fairness in the graph representations while improving efficiency and versatility.

\vspace{-.12in}
\subsection{Graph Coarsening}
We develop a new fairness-aware graph coarsening algorithm (shown in Algorithm \ref{alg:coarsen}). Given the initial graph $\mathcal{G}_{0}=\mathcal{G}$, it shrinks the graph size by collapsing a group of nodes into a supernode in the output graph $\mathcal{G}_{1}$. If two nodes are connected in $\mathcal{G}_{0}$, there exists an edge between their supernodes in $\mathcal{G}_{1}$. 
% To get more nodes matched, these algorithms usually process the nodes in increasing order of node degrees. 
As a result, the numbers of nodes and edges are reduced in $\mathcal{G}_{1}$. After repeating this process $c$ times, we can finally get the coarsened graph $\mathcal{G}_{c}$.

% In multi-level graph algorithms, the first step is to coarsen the input graph into a smaller one. Given the initial graph $\mathcal{G}_{0}$, we can shrink the graph size by collapsing two nodes into a supernode in the output graph $\mathcal{G}_{1}$. If two nodes are connected in $\mathcal{G}_{0}$, there exists an edge between their supernodes in $\mathcal{G}_{1}$. To get more nodes matched, these algorithms usually process the nodes in increasing order of node degrees. As a result, the numbers of nodes and edges are reduced in $\mathcal{G}_{1}$. After repeating this process $c$ times, we can finally get the coarsened graph $\mathcal{G}_{c}$. We adopt this idea and develop a new graph coarsening algorithm (shown in Algorithm \ref{alg:coarsen}) which takes fairness into consideration.

\begin{algorithm}[!t]
\caption{Graph coarsening for fair embedding}\label{alg:coarsen}
\begin{algorithmic}[1]
\REQUIRE Graph $\mathcal{G}_{i} = (\mathcal{V}_{i}, \mathcal{E}_{i})$
\ENSURE Coarsened graph $\mathcal{G}_{i+1}$
\STATE Sort $\mathcal{V}_{i}$ in an increasing order of node degrees
\FOR{unmatched $u \in \mathcal{V}_{i}$}
    \IF{All neighbors of $u$ are matched}
        \STATE $u' \gets \{u\}$
    \ELSE
        \STATE Find unmatched $v$ s.t. $(u, v) \in \mathcal{E}_{i}$ maximizing Equation (\ref{eqn:matching_policy}), then let $u' \gets \{u, v\}$
    \ENDIF
    \STATE Add supernode $u'$ to $\mathcal{V}_{i+1}$
\ENDFOR
\STATE Connect supernodes in $\mathcal{V}_{i+1}$ based on $\mathcal{E}_{i}$
\STATE Build $\mathcal{G}_{i+1}$ from $\mathcal{V}_{i+1}$ and $\mathcal{E}_{i+1}$
\end{algorithmic}
\end{algorithm}

The key challenges in graph coarsening are two-fold: retaining the structural information for better utility while incorporating fairness toward the sensitive attributes. For utility, we adopt a high-utility coarsening approach Normalized Heavy-Edge Matching (NHEM) \cite{karypis1998multilevel}, which merges two nodes if their normalized edge weight is maximum. Formally, given a node $u \in \mathcal{V}$, NHEM computes the normalized weight of the edge $(u, v) \in \mathcal{E}$ defined as:
\begin{equation}
    w(u, v) = \frac{\bm{A}_{u, v}}{\sqrt{\delta(u) \delta(v)}}
\end{equation}
% where $w(u, v)$ is the edge weight before normalization. 
Intuitively, it encourages nodes with fewer connections to match other nodes and penalizes the hub nodes. Node matching allows a supernode to have structural properties from different nodes. However, due to the homophily in many real-world networks, the bias in graph structure can be reinforced if two nodes from the same group are merged. 

In light of this, we add a new term in the node matching function to reflect the divergence of the sensitive attribute distributions of two nodes for fairness consideration. In \methodname{}, we use a matrix $\bm{S} \in \mathbb{R}^{N \times M}$ to denote the sensitive attribute values, where each row $s_{u}$ is the attribute value distribution vector for node $u$, and $M$ is the dimensionality of sensitive attribute values. We use these attribute distribution vectors to quantify the divergence between nodes.

Initially, each node in the original graph $\mathcal{G}_{0}$ has a concatenation of one-hot vectors for the attributes. For example, $\mathcal{F}$ has two attributes: `gender' (female, male) and `race' (African, Asian, white). The attributes of a female African $u$ can be modeled as $s_{u} = [\underline{1, 0}, \underline{1, 0, 0}]$. For a node $v$ in a coarser graph $\mathcal{G}_{i} (i > 0)$, we use $s_{v}$ to denote the distribution of all nodes in $\mathcal{G}_{0}$ merged into it. For example, $s_{v} = [\underline{1, 1}, \underline{1, 1, 0}]$ in $\mathcal{G}_{1}$ indicates that the supernode $v$ contains one female and one male on attribute `gender', and one African and one Asian on `race'. To measure the difference between the sensitive attribute distributions of two nodes, we define the following function based on Kullback–Leibler divergence \cite{kullback1951information}:
\begin{equation}
    \phi(u, v) = 1 - \left ( 1 + \sum_{j = 1}^{M} \frac{s_{u, j}}{\|s_{u}\|} \log \left( \frac{s_{u, j} / \|s_{u}\|}{s_{v, j} / \|s_{v}\|} \right) \right)^{-1}
\end{equation}
It essentially maps the divergence of two normalized attribute distributions to $[0, 1]$. The higher the score is, the more different their sensitive attributes are. Finally, given a node $u$ of $\mathcal{G}_{i}$, we formulate the node matching policy as follows:
\begin{equation}
    \max_{v: (u, v) \in \mathcal{E}_{i}} (1 - \lambda_{c}) w(u, v) + \lambda_{c} \phi(u, v)
\label{eqn:matching_policy}
\end{equation}

where $\lambda_{c}$ is the weight of fairness in node matching. The objective here is to find the neighbor of node $u$ that maximizes the edge weight and attribute divergence together. Intuitively, a large value of $\lambda_{c}$ generates more inter-group matching in graph coarsening. Section 5 will empirically show that Algorithm \ref{alg:coarsen} can improve the fairness in graph representations.

\subsection{Base Embedding}
% Following other multi-level frameworks \cite{karypis1998multilevel, liang2021mile}, \methodname{} applies the base model on the coarsest graph. Specifically, \methodname{} works in conjunction with an unsupervised graph embedding method such as DeepWalk \cite{perozzi2014deepwalk} or node2vec \cite{grover2016node2vec}. Note that our work is agnostic to the base model and requires no modification in their core implementations. This step generates the embeddings $\bm{E}_{c}$ on the coarsened graph.

Like other multi-level frameworks \cite{karypis1998multilevel, liang2021mile}, \methodname{} applies the base model on the coarsest graph in an agnostic manner. Since the input is the coarsened graph and the output is its node embeddings, it is straightforward that \methodname{} accommodates any unsupervised graph embedding method such as DeepWalk \cite{perozzi2014deepwalk} or node2vec \cite{grover2016node2vec} with no modification required. This step generates the embeddings $\bm{E}_{c}$ on the coarsened graph.

% \begin{figure}[!t]
%   \centering
%   \includegraphics[width=.95\linewidth]{images/Refinement.pdf}
%   \caption{The architecture of refinement module in \methodname{}. It first projects the embeddings of the coarser graph $\mathcal{G}_{i+1}$ to the nodes in $\mathcal{G}_{i}$ based on node matching results. Then it trains an $l$-layer GCN model to predict the refined embeddings using a unified fairness-aware objective. }
%   \label{fig:refinement}
% \end{figure}

\subsection{Refinement}
In the last phase of \methodname{}, we seek to learn the representations of graph $\mathcal{G}_{0}$ from 
% the set of coarsened graphs $\{\mathcal{G}_{1}, \mathcal{G}_{2}, \dots, \mathcal{G}_{c}\}$ and 
the embeddings of the coarsest graph $\mathcal{G}_{c}$. Generally, we train a fairness-aware refinement model based on graph convolution networks (GCN)~\cite{kipf2017semi} to infer the embeddings $\bm{E}_{c-1}$ of $\mathcal{G}_{c-1}$ from $\bm{E}_{c}$. Then we iteratively apply it until we get $\bm{E}_{0}$. 

\subsubsection{Model Architecture}
Figure \ref{fig:refinement} shows the architecture of our refinement model.
Without loss of generality, the refinement model has a projection layer followed by $l$ GNN layers where the input and output of layer $i \in [1, l]$ are denoted as $\bm{H}_{i-1}$ and $\bm{H}_{i}$. Given two graphs $\mathcal{G}_{j}$ and $\mathcal{G}_{j+1}$, we initialize $\bm{H}_{0}$ by projecting the embeddings of supernodes in $\mathcal{G}_{j+1}$ to its associated nodes in $\mathcal{G}_{j}$. Note that if two nodes that have different sensitive attribute values are merged into $\mathcal{G}_{j+1}$, they share the same initial embeddings in $\mathcal{G}_{j}$, which mitigates the potential bias in learned representations. 

In each layer, we use a normalized adjacency matrix $\tilde{\bm{D}}^{-\frac{1}{2}} \tilde{\bm{A}} \tilde{\bm{D}}^{-\frac{1}{2}} $ for message passing. Here we drop the notation referring to a specific graph for clarity. To take the sensitive attributes into consideration, we concatenate the input of each layer with the row-normalized sensitive attribute matrix $\tilde{\bm{S}}$. Formally, the $i$-th convolution layer in our refinement model can be formulated as
\begin{equation}
    \bm{H}_{i} = \mathrm{tanh} \left ( \tilde{\bm{D}}^{-\frac{1}{2}} \tilde{\bm{A}} \tilde{\bm{D}}^{-\frac{1}{2}}~(\bm{H}_{i-1} \parallel \tilde{\bm{S}})~\Theta_{i} \right )
\end{equation}
where $\Theta_{i}$ is the trainable linear transformation matrix of layer $i$. Finally, we can infer $\mathcal{G}_{j}$'s node representations $\bm{E}_{j} = \bm{H}_{l} / \| \bm{H}_{l} \|_{2}$.

\subsubsection{Training objectives}
To achieve a trade-off between utility and fairness, we have two distinct objectives for each of them. For utility, we expect the refined representations to be close to the input so that matched nodes are still close to each other in the embedding space after refinement. Therefore, we minimize the difference between the projected embeddings and the predicted ones generated by the refinement model, which is defined as
\begin{equation}
    L_{u} = \frac{1}{|\mathcal{V}_{c}|} \| \bm{H}_{0} - \bm{H}_{l} \|^{2}
\end{equation}

$L_{u}$ is the utility loss in our framework. To improve fairness, we encourage nodes with different sensitive attribute values to be closer in the embedding space. 
%As a result, the fairness is improved in terms of common metrics like $\Delta_{DP}$ and $\Delta_{EO}$. 
Specifically, we create a subset of edges $\mathcal{E}'_{c}$ that consists of links between two nodes with significantly different sensitive attributes, which can be formulated as
\begin{equation}
    \mathcal{E}'_{c} = \{(u, v): (u, v) \in \mathcal{E}_{c} \wedge \phi(u, v) \geq \gamma \} \nonumber
\end{equation}
where $\gamma$ is a threshold parameter for attribute divergence. Then we use the learned representations to reconstruct an adjacency matrix that represents the node distances in the embedding space. To reward the links between diverse attribute groups, we use the Hadamard product of the reconstruction matrix and the adjacency matrix of $\mathcal{E}'_{c}$ (denoted as $\mathcal{A}'_{c}$) as our fairness objective. As a result, the trained model generates similar embeddings for inter-group nodes. The fairness loss ($L_{f}$)  is formally defined as
\begin{equation}
    L_{f} = - \frac{1}{|\mathcal{E}'_{c}|} \left [ 
    \mathrm{sigmoid}(\bm{H}_{l}{\bm{H}_{l}}^{\top}) \odot \mathcal{A}'_{c} \right ]
\label{eqn:fairloss}
\end{equation}

The multiplication of the embedding matrix with a sigmoid activation (denoted by $\mathrm{sigmoid}(\bm{H}_{l}{\bm{H}_{l}}^{\top})$) reconstructs the entire $N \times N$ matrix. Its element-wise product with $\mathcal{A}'_{c}$ performs the masking operation to select the node pairs with diverse sensitive attributes. The negative minimization of $L_{f}$ ensures that the similarity of the selected node pairs -- measured by the dot products of their embeddings -- is increased. This loss ensures that the embeddings learned by the refinement model of the node pairs with diverse sensitive attributes are {\it proximal} to each other.

Combining the utility and fairness loss functions, the overall training objective of our refinement model is 
\begin{equation}
    \min_{\{\Theta_{i}, \forall i \in [1,~l]\}} (1 - \lambda_{r}) L_{u} + \lambda_{r} L_{f}
\end{equation}
where $\lambda_{r} \in [0, 1]$ controls the weight of fairness objective.

%%
%% Theoretical analysis

\subsection{Theoretical analysis}
\label{sec:analysis}
\textbf{Time complexity:} The time complexity of \methodname{} depends on the selected embedding approach. Note that such approaches typically have a time complexity of at least $O(d|\mathcal{V}_{0}|)$~\cite{cui2018survey}, for example, the time complexity of DeepWalk is $O(d|\mathcal{V}_{0}|\log |\mathcal{V}_{0}|)$. Considering that the number of nodes can be reduced by up to half after each time of coarsening (observed in Section \ref{sec:ablation}), the efficiency is significantly improved by embedding the coarsened graph. Apart from embedding the coarsened graph, \methodname{} spends additional $O \left ( cl(d+M) \left ( |\mathcal{E}_{0}| + d |\mathcal{V}_{0}| \right ) \right )$ time on coarsening and refinement. Given that $d, M \ll |\mathcal{V}_{0}|$, the additional time of these two phases is typically much less than the reduced time of base embedding - empirically verified in Section \ref{sec:ablation}. 
Additional details on complexity analysis are included in our supplement. % see \autoref{sec:time_analysis}.

\noindent \textbf{Fairness:} We prove that the difference between the mean representations of different demographic groups is bounded depending on the network topology. % Proof is provided in Appendix \ref{sec:time_analysis}.
\begin{theorem}
When $L_f$ is minimized, the 2-norm of the difference between the mean embeddings of any two demographic groups regarding a given sensitive attribute is bounded by
\begin{equation}
    \left \| \bm{\mu}_{p} - \bm{\mu}_{q} \right \|_{2} \leq 2 (1 - \min(\beta_{p}, \beta_{q})) \nonumber
% \label{eqn:fairness_theorem_main}
\end{equation}
where $p, q$ are any two different values of the given sensitive attribute. For $i \in \{p, q\}$, $\bm{\mu}_{i}$ denotes the mean embedding values of nodes from group $i$, and $\beta_{i}$ denotes the ratio of nodes from group $i$ that have at least one inter-group edge. % Note that this applies to any two values for a multi-class sensitive attribute without loss of generality.
\label{thm:fairness_main}
\end{theorem}

%%%%
% Datasets here.
%%%%

% \begin{table*}[!th]
% \caption{Statistics of datasets.}
% \label{table:dataset}
% \centering
% % \begin{tabular}{l|llll}
% % \toprule
% % Dataset & \makecell[c]{German} & \makecell[c]{Recidivism} & \makecell[c]{Credit} & \makecell[c]{Pokec-n} \\
% % \midrule
% % \# Nodes & 1,000 & 18,876 & 30,000 & 66,569 \\
% % % \midrule
% % \# Edges & 22,242 & 321,308 & 1,436,858 & 729,129 \\
% % % \midrule
% % \# Features & 27 & 18 & 13 & 59 \\
% % % \midrule
% % \multirow{2}{*}{Label ($|\mathcal{Y}|$)} & credit & \multirow{2}{*}{crime (2)} & payment & working \\
% %  & risk (2) & & default (2) & field (4) \\
% % % \midrule
% % Sens. Attr.  & \multirow{2}{*}{gender (2)} & \multirow{2}{*}{race (2)} & \multirow{2}{*}{age (2)} & region (2) \\
% % ($|\mathcal{S}|$) & & & & gender (2) \\
% % \bottomrule
% % \end{tabular}
% \begin{tabular}{l|lllll}
% \toprule
% Dataset & \# Nodes & \# Edges & \# Features & Label ($|\mathcal{Y}|$) & Sensitive Attributes ($|\mathcal{S}|$) ~\\
% \midrule
% German & 1,000 & 22,242 & 27 & credit risk (2) & gender (2) \\
% Recidivism & 18,876 & 321,308 & 18 & crime (2) & race (2) \\
% Credit & 30,000 & 1,436,858 & 13 & payment default (2) & age (2) \\
% Pokec-n & 66,569 & 729,129 & 59 & working field (4) & region (2), gender (2) \\
% \bottomrule
% \end{tabular}
% \end{table*}

\begin{table*}[!t]
\caption{Statistics of datasets (NC denotes Node Classification, and LP denotes Link Prediction).}
\label{table:dataset}
\vspace{-0.1in}
\begin{center}
\begin{small}
% \resizebox{0.80\linewidth}{!}{
% \begin{tabular}{l|lllllll}
% \toprule
% Dataset & Task & \# Nodes & \# Edges & \# Features & Label ($|\mathcal{Y}|$) & Sensitive Attributes ($|\mathcal{S}|$) & Citation ~\\
% \midrule
% German & NC & 1,000 & 22,242 & 27 & credit risk (2) & gender (2) & \cite{agarwal2021towards, dong2022edits, zhang2021multi} \\
% Recidivism & NC & 18,876 & 321,308 & 18 & crime (2) & race (2) & \cite{agarwal2021towards, dong2022edits, ma2022learning, zhang2021multi} \\
% Credit & NC & 30,000 & 1,436,858 & 13 & payment default (2) & age (2) & \cite{agarwal2021towards, dong2022edits, ma2022learning, zhang2021multi} \\
% Pokec-n & NC & 66,569 & 729,129 & 59 & field of work (4) & region (2), gender (2) & \cite{dai2021say, dong2022edits, franco2022deep} \\
% Cora & LP & 2,708 & 5,278 & 1433 & citation (2) & paper category (7) & \cite{sen2008collective, li2020dyadic, spinelli2021fairdrop, current2022fairmod} \\
% Citeseer & LP & 3,312 & 4,660 & 3703 & citation (2) & paper category (6) & \cite{sen2008collective, li2020dyadic, spinelli2021fairdrop, current2022fairmod} \\
% Pubmed & LP & 19,717 & 44,338 & 500 & citation (2) & paper category (3) & \cite{namata2012query, li2020dyadic, spinelli2021fairdrop, current2022fairmod} \\
% \bottomrule
% \end{tabular}

\begin{tabular}{l|llllll}
\toprule
Dataset & Task & \# Nodes & \# Edges & \# Features & Label ($|\mathcal{Y}|$) & Sensitive Attributes ($|S|$) ~\\
\midrule
German & NC & 1,000 & 22,242 & 27 & credit risk (2) & gender (2) \\
Recidivism & NC & 18,876 & 321,308 & 18 & crime (2) & race (2) \\
Credit & NC & 30,000 & 1,436,858 & 13 & payment default (2) & age (2) \\
Pokec-n & NC & 66,569 & 729,129 & 59 & field of work (4) & region (2), gender (2) \\
Cora & LP & 2,708 & 5,278 & 1433 & citation (2) & paper category (7) \\
Citeseer & LP & 3,312 & 4,660 & 3703 & citation (2) & paper category (6) \\
Pubmed & LP & 19,717 & 44,338 & 500 & citation (2) & paper category (3) \\
\bottomrule
\end{tabular}
\end{small}
\end{center}
% }
\vspace{-0.12in}
\end{table*}

Theorem \ref{thm:fairness_main} shows that the difference between the mean embeddings of two groups depends on the ratio of inter-group connected nodes in each group, which is typically large. For example, among the datasets used in our experiments, the minimum $\beta$ is $0.676$ in Credit and $0.958$ in German, respectively. When the mean embeddings of different demographic groups are close to each other, they have similar representations and therefore with a high likelihood, they will receive similar outcomes in the downstream task. The full proof of Theorem \ref{thm:fairness_main} is provided in our supplement.

% The fairness of learned embeddings is dependent on the coarsened graph from the first phase and the refinement model trained on it. When the original graph is coarsened multiple times, a supernode represents a large group of nodes from different demographic groups. Our heuristic method is to initiate these nodes with the same embedding so that they are eventually similar in the embedding space. In Section \ref{sec:ablation}, we will observe the effectiveness of this method in improving fairness. 

% In Section \ref{sec:ablation}, we will empirically study \textbf{P1:} how the graph size changes after each time of coarsening, \textbf{P2:} how \methodname{} improves the efficiency by coarsening the graph, \textbf{P3:} how the fairness evolves with varying $c$, and \textbf{P4:} how the coarsened graph affects the utility of learned embeddings.

%% file: sections/experiments.tex
\label{sec:experiments}
\subsection{Experiment Setup}
\noindent \textbf{Datasets}:
% We adopt the commonly used node classification task for evaluation. 
We examine the performance of \methodname{} on both node classification and link prediction tasks. Our experiments are conducted on seven real-world datasets from different application scenarios widely used in the fairness literature~\cite{agarwal2021towards, dai2021say, dong2022edits, li2020dyadic, spinelli2021fairdrop}. Statistics of the datasets are shown in \autoref{table:dataset}, where $|\mathcal{Y}|$ denotes the number of predicted labels, and $|S|$ denotes the number of sensitive attribute values. For details, please refer to our supplement.% \autoref{sec:dataset_description}. % these literature~\cite{agarwal2021towards,takac2012data,sen2008collective,namata2012query}.

\noindent \textbf{Metrics}: To quantify the prediction performance in node classification, we use AUROC, F1-score (for binary class problems), and Micro-F1 (for multi-class problems) as our utility metric. To measure the group fairness, we use $\Delta_{DP}$ and $\Delta_{EO}$ described in Section 3 as our fairness metrics. We also report the end-to-end running time in seconds to show the efficiency of all methods.

For link prediction, following prior works~\cite{li2020dyadic, spinelli2021fairdrop}, we use AUROC, Average Precision (AP), and accuracy as the utility metrics, and compute the disparity in expected prediction scores between intra-group and inter-group links. Specifically, the fairness metrics are formulated as:
\begin{equation}
\begin{split}
    \Delta_{DP,~\mathrm{LP}} = & ~|\mathbb{E}_{(u, v) \sim \mathcal{V} \times \mathcal{V}} [\hat{Y} | S(u) = S(v)]  \nonumber \\
    & - \mathbb{E}_{(u, v) \sim \mathcal{V} \times \mathcal{V}} [\hat{Y} | S(u) \neq S(v)]| \nonumber \\
    \Delta_{EO,~\mathrm{LP}} = & ~|\mathbb{E}_{(u, v) \sim \mathcal{V} \times \mathcal{V}} [\hat{Y} | (u, v) \in \mathcal{E}, S(u) = S(v)] \nonumber \\
    & - \mathbb{E}_{(u, v) \sim \mathcal{V} \times \mathcal{V}} [\hat{Y} | (u, v) \in \mathcal{E}, S(u) \neq S(v)]| \nonumber
\end{split}
\end{equation}
where $\hat{Y} \in [0, 1]$ is the link prediction score. % is the prediction score regarding the existence of an edge between two given nodes.

\noindent \textbf{Baselines}: Our baselines include 1) \textit{Specialized approaches}: For node classification, we use the vanilla GCN~\cite{kipf2017semi} and three state-of-the-art fair node classification methods (NIFTY~\cite{agarwal2021towards}, FairGNN~\cite{dai2021say}, and EDITS~\cite{dong2022edits}) with GCN as their base model. 
For link prediction, we use VGAE~\cite{kipf2016variational} and FairAdj~\cite{li2020dyadic} with VGAE as the base predictor in our comparative experiments. In addition, we adapt CFGE~\cite{bose2019compositional} for both tasks which is the only baseline that accommodates multiple sensitive attributes. 2) \textit{Graph embedding approaches}: We choose three popular unsupervised graph embedding techniques: NetMF~\cite{qiu2018network}, DeepWalk~\cite{perozzi2014deepwalk}, and Node2vec~\cite{grover2016node2vec}. We also include FairWalk~\cite{rahman2019fairwalk}, which is essentially a fairness-aware adaption of Node2vec. Note that these approaches can be used in both downstream tasks. 3) \textit{Our framework}: We let \methodname{} run with the three embedding approaches above for evaluation.

\begin{table*}[!t]
\caption{Comparison on utility, fairness, and efficiency metrics in node classification between \methodname{} and other baselines.}
\label{table:overall_results}
\vspace{-0.1in}
\begin{center}
\begin{small} 
\begin{tabular}{c|l|ll|rr|r}
\toprule
Dataset & \makecell[c]{Method} & \makecell[c]{AUROC ($\uparrow$)} & \makecell[c]{F1 ($\uparrow$)} & \makecell[c]{$\Delta_{DP}~(\downarrow)$} & \makecell[c]{$\Delta_{EO}~(\downarrow)$} & \makecell[c]{Time $(\downarrow)$}\\
\midrule
\multirow{11}{*}{German}
 & NetMF & \textbf{65.16 $\pm$ 2.45} & 80.63 $\pm$ 1.10 & 5.71 $\pm$ 2.89 & 3.66 $\pm$ 2.11 & \textbf{2.48} \\
 & \methodname{}-NetMF & 61.93 $\pm$ 3.38 & \textbf{82.35 $\pm$ 0.00} & \textbf{0.00 $\pm$ 0.00} & \textbf{0.00 $\pm$ 0.00} & 6.31 \\
 \cmidrule{2-7}
 & DeepWalk & 58.54 $\pm$ 4.43 & 75.78 $\pm$ 1.49 & 7.22 $\pm$ 3.86 & 7.69 $\pm$ 3.26 & 16.99 \\
 & \methodname{}-DeepWalk & \textbf{63.31 $\pm$ 3.63} & \textbf{82.40 $\pm$ 0.33} & \textbf{0.67 $\pm$ 0.88} & \textbf{0.26 $\pm$ 0.39} & \textbf{7.84} \\
\cmidrule{2-7}
 & Node2vec & 63.37 $\pm$ 3.77 & 78.69 $\pm$ 1.25 & 3.69 $\pm$ 2.60 & 2.75 $\pm$ 1.34 & 12.76 \\
 & Fairwalk & \textbf{63.98 $\pm$ 2.07} & 77.64 $\pm$ 1.62 & 3.67 $\pm$ 2.74 & 3.28 $\pm$ 2.50 & 11.93 \\
 & \methodname{}-Node2vec & 62.00 $\pm$ 2.59 & \textbf{82.32 $\pm$ 0.20} & \textbf{0.60 $\pm$ 0.96} & \textbf{0.44 $\pm$ 0.41} & \textbf{8.29} \\
 \cmidrule{2-7}
 & Vanilla GCN & 64.75 $\pm$ 7.20 & 77.93 $\pm$ 3.53 & 16.27 $\pm$ 5.86 & 13.28 $\pm$ 5.06 & \textbf{23.75} \\
 & FairGNN & 53.12 $\pm$ 5.73 & \textbf{82.35 $\pm$ 0.00} & \textbf{0.00 $\pm$ 0.00} & \textbf{0.00 $\pm$ 0.00} & 136.29 \\
 & NIFTY & 56.65 $\pm$ 6.84 & 81.35 $\pm$ 1.54 & 1.20 $\pm$ 1.45 & 0.83 $\pm$ 1.20 & 91.05 \\
 & EDITS & \textbf{64.93 $\pm$ 2.90} & 79.64 $\pm$ 2.27 & 5.00 $\pm$ 3.38 & 2.76 $\pm$ 2.26 & 84.24 \\
 & CFGE & 64.38 $\pm$ 0.77 & 81.59 $\pm$ 0.33 & 4.54 $\pm$ 2.91 & 4.30 $\pm$ 1.84 & 3900.11 \\
\midrule
% \multirow{11}{*}{Recidivism}
%  & NetMF & \textbf{94.63 $\pm$ 0.17} & \textbf{85.46 $\pm$ 0.29} & 3.41 $\pm$ 0.21 & 1.62 $\pm$ 0.78 & 141.90 \\
%  & \methodname{}-NetMF & 89.52 $\pm$ 0.50 & 77.65 $\pm$ 0.47 & \textbf{2.81 $\pm$ 0.50} & \textbf{0.75 $\pm$ 0.55} & \textbf{29.66} \\
% \cmidrule{2-7}
%  & DeepWalk & \textbf{93.33 $\pm$ 0.35} & \textbf{83.62 $\pm$ 0.42} & 3.47 $\pm$ 0.37 & 1.28 $\pm$ 0.60 & 303.68 \\
%  & \methodname{}-DeepWalk & 86.93 $\pm$ 0.74 & 73.50 $\pm$ 0.99 & \textbf{2.71 $\pm$ 0.58} & \textbf{1.08 $\pm$ 0.77} & \textbf{45.93} \\
% \cmidrule{2-7}
%  & Node2vec & \textbf{92.56 $\pm$ 0.26} & \textbf{83.31 $\pm$ 0.36} & 3.61 $\pm$ 0.56 & 1.57 $\pm$ 0.97 & 136.33 \\
%  & Fairwalk & 92.43 $\pm$ 0.43 & 82.99 $\pm$ 0.51 & 3.32 $\pm$ 0.24 & 1.48 $\pm$ 0.66 & 133.62 \\
%  & \methodname{}-Node2vec & 87.00 $\pm$ 0.50 & 71.34 $\pm$ 0.86 & \textbf{2.75 $\pm$ 0.35} & \textbf{1.15 $\pm$ 0.65} & \textbf{38.67} \\
% \cmidrule{2-7}
%  & Vanilla GCN & \textbf{88.16 $\pm$ 1.72} & \textbf{77.68 $\pm$ 1.63} & 3.83 $\pm$ 0.59 & 1.46 $\pm$ 0.71 & \textbf{474.57} \\
%  & FairGNN & 67.26 $\pm$ 7.80 & 44.63 $\pm$ 14.87 & \textbf{0.67 $\pm$ 0.45} & 1.24 $\pm$ 0.40 & 1071.39 \\
%  & NIFTY & 77.89 $\pm$ 4.21 & 64.44 $\pm$ 6.11 & 1.34 $\pm$ 1.01 & \textbf{0.63 $\pm$ 0.42} & 1651.09 \\
%  & EDITS & 79.48 $\pm$ 13.26 & 69.66 $\pm$ 13.28 & 4.39 $\pm$ 2.10 & 2.52 $\pm$ 2.04 & 1311.42 \\
% & CFGE & 60.92 $\pm$ 1.88 & 25.58 $\pm$ 6.45 & 0.81 $\pm$ 0.58 & 1.45 $\pm$ 0.88 & 2498.52 \\
% \midrule
\multirow{11}{*}{Credit}
 & NetMF & \textbf{74.93 $\pm$ 0.43} & \textbf{88.36 $\pm$ 0.08} & 2.66 $\pm$ 0.55 & 1.34 $\pm$ 0.93 & 240.38 \\
 & \methodname{}-NetMF & 74.69 $\pm$ 0.43 & 88.31 $\pm$ 0.08 & \textbf{0.68 $\pm$ 0.50} & \textbf{0.68 $\pm$ 0.66} & \textbf{90.28} \\
\cmidrule{2-7}
 & DeepWalk & \textbf{75.09 $\pm$ 0.39} & \textbf{88.36 $\pm$ 0.15} & 2.50 $\pm$ 0.54 & 1.81 $\pm$ 0.71 & 570.35 \\
 & \methodname{}-DeepWalk & 74.60 $\pm$ 0.53 & 88.31 $\pm$ 0.11 & \textbf{1.34 $\pm$ 0.57} & \textbf{1.05 $\pm$ 0.70} & \textbf{112.39} \\
\cmidrule{2-7}
 & Node2vec & \textbf{74.95 $\pm$ 0.36} & \textbf{88.26 $\pm$ 0.18} & 2.92 $\pm$ 0.41 & 1.93 $\pm$ 0.87 & 249.92 \\
 & Fairwalk & 74.88 $\pm$ 0.49 & 88.25 $\pm$ 0.15 & 2.78 $\pm$ 0.47 & 1.58 $\pm$ 0.81 & 261.23 \\
 & \methodname{}-Node2vec & 74.01 $\pm$ 0.47 & 87.99 $\pm$ 0.16 & \textbf{1.47 $\pm$ 0.74} & \textbf{1.33 $\pm$ 0.64} & \textbf{105.43} \\
\cmidrule{2-7}
 & Vanilla GCN & 72.83 $\pm$ 2.83 & 82.79 $\pm$ 1.76 & 5.77 $\pm$ 0.32 & 4.36 $\pm$ 0.59 & \textbf{2265.58} \\
 & FairGNN & 71.14 $\pm$ 4.58 & 83.29 $\pm$ 3.27 & \textbf{1.40 $\pm$ 0.97} & \textbf{1.21 $\pm$ 0.50} & 3201.72 \\
 & NIFTY & 72.43 $\pm$ 0.68 & 81.80 $\pm$ 0.37 & 5.52 $\pm$ 0.33 & 4.35 $\pm$ 0.68 & 7735.00 \\
 & EDITS & \textbf{73.22 $\pm$ 0.77} & 81.55 $\pm$ 0.30 & 5.36 $\pm$ 0.42 & 4.17 $\pm$ 0.65 & 9078.12 \\
 & CFGE & 70.51 $\pm$ 1.31 & \textbf{87.85 $\pm$ 0.13} & 2.99 $\pm$ 0.47 & 1.47 $\pm$ 0.18 & 12430.37 \\
\bottomrule
\end{tabular}
\end{small}
\end{center}
\vspace{-0.2in}
\end{table*}

% To demonstrate the improvement clearly, we leverage PCA to visualize the learned representations in a 2-dimensional space. We choose NetMF and \methodname{}-NetMF ($c=2$) as an example and plot the projected representations in \autoref{fig:visual}. In Figure \ref{fig:visual_netmf}, we can see that nodes mostly cluster with their demographic group, which indicates that the representations learned by NetMF are biased about the sensitive attribute. In contrast, \methodname{} generates less biased embeddings. As shown in Figure \ref{fig:visual_fairmile}, nodes from different groups scatter in the same area, therefore it becomes difficult to distinguish them.
% \begin{figure}[!t]
% \centering
% \subfloat[Vanilla NetMF]{\includegraphics[width=0.475\linewidth]{images/german_netmf_128_20.pdf}%
% \label{fig:visual_netmf}}
% \subfloat[FairMILE-NetMF]{\includegraphics[width=0.475\linewidth]{images/german_FairMILE-2-netmf_128_20.pdf}%
% \label{fig:visual_fairmile}}
% \caption{Visualization of Embeddings by PCA.}
% \label{fig:visual}
% \end{figure}

\noindent \textbf{Parameters and environments}:
All hyperparameters of specialized approaches are set following the authors' instructions. In particular, unless otherwise specified, we set $\lambda=50$ in CFGE for a better tradeoff between fairness and utility. The dimensionality $d$ of representations for graph embedding approaches and our work is set to 128. In \methodname{}, we set $\lambda_{c}=0.5$ for graph coarsening. For refinement, We train a two-layer model for 200 epochs with $\lambda_{r}=0.5$, $\gamma=0.5$, and the learning rate of $1 \times 10^{-3}$ on all datasets. 
% The refinement model is trained with the Adam optimizer~\cite{kingma2014adam}.
To evaluate unsupervised approaches for node classification and all approaches for link prediction, we train a linear classifier on the learned embeddings (for node classification) or the Hadamard products of embeddings of sampled node pairs (for link predictions). For the task of link prediction, we randomly sample $10\%$ of edges to build test sets and remove them from the training data, then we add the same number of negative samples as positive edges in the training and test sets, respectively. All methods are evaluated on the same test sets and trained on CPUs for fair comparisons. Experiments are conducted on a Linux machine with a 28-core Intel Xeon E5-2680 CPU and 128GB RAM. We report the average results of 5 runs with different data splits. For reproducibility, our codes and data are available\footnote{\url{https://github.com/heyuntian/FairMILE}}. % \footnote{\url{https://www.dropbox.com/sh/tvz87kuhkyg00sq/AAAoQXQ-J4NMgUWSXM_cZ6xHa}}. 

\subsection{Results for Node Classification}
\label{sec:exp_nc}
We first compare our work with specialized GNN-based approaches and unsupervised graph embedding baselines on three datasets with a single binary sensitive attribute. Specifically, we address the following questions: \textbf{Q1:} Does \methodname{} improve the fairness of a base embedding method? How is the fairness of \methodname{} compared with specialized approaches? \textbf{Q2:} Does \methodname{} outperform other baselines in terms of efficiency? \textbf{Q3:} Does \methodname{} retain the embedding's utility while improving its fairness and efficiency?
% Due to the space limit,
\autoref{table:overall_results} only presents the results on two datasets (German and Credit). The results on Recidivism are included in our supplement. %\autoref{sec:full_result_node_classification}. 
Methods are categorized into different groups by their base models, and the optimal performance in each group is highlighted in bold.

\noindent \textbf{A1: \methodname{} achieves better fairness scores.} Compared with graph embedding approaches without fairness consideration (i.e., NetMF, DeepWalk, and Node2vec), \methodname{} on top of them always has lower $\Delta_{DP}$ and $\Delta_{EO}$. % \methodname{} always has lower $\Delta_{DP}$ and $\Delta_{EO}$ than existing graph embedding methods without fairness considerations. 
In German, \methodname{} decreases the $\Delta_{DP}$ scores of Node2vec and DeepWalk by $83.7\%$ and $90.7\%$, respectively. In terms of $\Delta_{EO}$, \methodname{} improves the fairness of Node2vec and DeepWalk by $84.0\%$ and $96.6\%$. When choosing NetMF as the base model, \methodname{} is optimal (zero) on both fairness metrics, which indicates that the learned representations lead to a perfectly fair classification. Compared with FairWalk (the fair adoption of Node2vec), \methodname{}-Node2vec improves $\Delta_{DP}$ by $49.7\%$ on Credit while FairWalk only improves by $4.8\%$. Similar results are also observed on other datasets, which reveals that \methodname{} successfully mitigates the bias in general-purpose graph embedding baselines. 

% \textbf{\methodname{} achieves comparable fairness w.r.t. SOTA.} 
% and a fairness-aware graph embedding technique (FairWalk) 
% On Recidivism, it still outperforms FairGNN and EDITS on $\Delta_{DP}$ and $\Delta_{EO}$. 
% As an adaption of Node2vec, FairWalk only slightly improves $\Delta_{DP}$ by up to $8.0\%$ (in Recidivism), while \methodname{}-Node2vec improves by up to $83.7\%$ (in German). 

We also evaluate specialized methods (i.e., FairGNN, NIFTY, EDITS, and CFGE) and observe that \methodname{} achieves comparable or better fairness with respect to these approaches. FairGNN outperforms the other specialized methods on all datasets by gaining the largest improvements on both $\Delta_{DP}$ and $\Delta_{EO}$ with respect to Vanilla GCN. It means that their methodology of adversarial training is more effective than the regularization adopted by NIFTY, and EDITS suffers from its task agnosticity. On German and Credit, the best fairness metrics of \methodname{} are comparable to or better than FairGNN and other specialized techniques. These results demonstrate that \methodname{} is effective in reducing the bias in graph representations compared to the state-of-the-art models. % , at a fraction of the cost (from an efficiency perspective). 

% When we look at larger datasets (i.e., Recidivism and Credit)
\noindent \textbf{A2: \methodname{} is more efficient than other baselines.} First, \methodname{} outperforms the GNN-based specialized approaches on all datasets in terms of efficiency. In German, they take up to 2 minutes for training, while \methodname{} can finish in only 6 seconds. The difference becomes more significant on larger datasets. When CFGE needs more than 3 hours on Credit, \methodname{} finishes within 2 minutes which is $110.6$-$137.7\times$ faster. On the other hand, \methodname{} also improves the efficiency of all base embedding methods. In Credit, \methodname{} on top of NetMF saves up to $80\%$ of the running time of vanilla NetMF. The improvement in German is sometimes invisible because German is a small graph. But \methodname{} can still finish within seconds. Compared with FairWalk, \methodname{}-Node2vec is always faster.

\noindent \textbf{A3: \methodname{} learns quality graph representations.} We observe the quality of learned representations through the AUROC and F1 scores. With respect to the base embedding methods, \methodname{} has a similar performance on both utility metrics which is fairly remarkable given that \methodname{} significantly improves fairness. An interesting observation is that while Vanilla GCN, graph embedding approaches, and FairMILE have a similar performance in terms of utility, the supervised GCN-based fair approaches have lower AUROC or F1 scores on one or both datasets, which reveals that these approaches require fine-tuning to perform well in terms of utility while enforcing fairness.

In summary, \methodname{} in conjunction with popular base embedding approaches can compete or improve on the fairness criteria with various specialized methods while outperforming them significantly in terms of efficiency and retaining comparable utility. % In other words compared with state-of-the-art fair representation learning techniques, \methodname{} achieves similar or better performance on utility and fairness, while significantly outperforming them in terms of efficiency.

%%%%
% Multiple sensitive attributes
%%%%

\begin{table*}[t]
\caption{Case study on Pokec-n consisting of multiple sensitive attributes.}
\label{table:pokec}
\vspace{-0.1in}
\begin{center}
\begin{small} 
\begin{tabular}{l|cc|cc|cc|r}
\toprule
\makecell[c]{\multirow{2}{*}{Method}} & \multirow{2}{*}{AUROC ($\uparrow$)} & \multirow{2}{*}{Micro-F1 ($\uparrow$)} & \multicolumn{2}{c|}{Region} & \multicolumn{2}{c|}{Gender} & \makecell[c]{\multirow{2}{*}{Time $(\downarrow)$}} \\
\cline{4-7}
 & & & $\Delta_{DP}~(\downarrow)$ & $\Delta_{EO}~(\downarrow)$ & $\Delta_{DP}~(\downarrow)$ & $\Delta_{EO}~(\downarrow)$ & \\
\midrule
% NetMF & \textbf{73.12 $\pm$ 0.46} & \textbf{62.18 $\pm$ 0.37} & 0.26 $\pm$ 0.14 & 0.44 $\pm$ 0.25 & 1.52 $\pm$ 0.40 & 0.98 $\pm$ 0.39 & 1542.78 \\
% \methodname{}-NetMF & 68.86 $\pm$ 0.91 & 55.96 $\pm$ 0.67 & \textbf{0.22 $\pm$ 0.09} & \textbf{0.43 $\pm$ 0.12} & \textbf{0.93 $\pm$ 0.23} & \textbf{0.54 $\pm$ 0.25} & \textbf{204.99} \\
% \hline
% DeepWalk & \textbf{72.58 $\pm$ 0.67} & \textbf{61.44 $\pm$ 0.68} & 0.44 $\pm$ 0.16 & 0.47 $\pm$ 0.21 & 1.39 $\pm$ 0.57 & 1.37 $\pm$ 0.40 & 1139.30 \\
% \methodname{}-DeepWalk & 67.43 $\pm$ 1.16 & 54.98 $\pm$ 0.88 & \textbf{0.18 $\pm$ 0.16} & \textbf{0.33 $\pm$ 0.18} & \textbf{0.64 $\pm$ 0.25} & \textbf{0.57 $\pm$ 0.24} & \textbf{352.38} \\
% \hline
% Node2vec & 72.71 $\pm$ 0.46 & 61.71 $\pm$ 0.32 & 0.38 $\pm$ 0.06 & 0.53 $\pm$ 0.12 & 1.79 $\pm$ 0.50 & 1.12 $\pm$ 0.47 & 798.01 \\
% Fairwalk (Region) & \textbf{72.76 $\pm$ 0.42} & \textbf{61.87 $\pm$ 0.28} & 0.37 $\pm$ 0.13 & 0.47 $\pm$ 0.25 & 1.37 $\pm$ 0.56 & 1.12 $\pm$ 0.39 & 838.19 \\
% Fairwalk (Gender) & 72.47 $\pm$ 0.57 & 61.69 $\pm$ 0.62 & 0.41 $\pm$ 0.17 & 0.51 $\pm$ 0.22 & 1.33 $\pm$ 0.56 & 1.08 $\pm$ 0.47 & 828.11 \\
% \methodname{}-Node2vec & 70.10 $\pm$ 0.61 & 57.77 $\pm$ 0.92 & \textbf{0.32 $\pm$ 0.08} & \textbf{0.20 $\pm$ 0.11} & \textbf{0.99 $\pm$ 0.32} & \textbf{0.79 $\pm$ 0.38} & \textbf{230.85} \\
NetMF & 73.12 $\pm$ 0.46 & 62.18 $\pm$ 0.37 & 0.26 $\pm$ 0.15 & 1.82 $\pm$ 0.48 & 0.39 $\pm$ 0.24 & 1.04 $\pm$ 0.58 & 1542.78 \\
\methodname{}-NetMF & 68.86 $\pm$ 0.91 & 55.96 $\pm$ 0.67 & \textbf{0.17 $\pm$ 0.07} & \textbf{1.18 $\pm$ 0.29} & \textbf{0.31 $\pm$ 0.06} & \textbf{0.56 $\pm$ 0.30} & \textbf{204.99} \\
\midrule
DeepWalk & 72.58 $\pm$ 0.67 & 61.44 $\pm$ 0.68 & 0.38 $\pm$ 0.07 & 1.49 $\pm$ 0.63 & 0.45 $\pm$ 0.22 & 1.61 $\pm$ 0.47 & 1139.30 \\
\methodname{}-DeepWalk & 67.43 $\pm$ 1.16 & 54.98 $\pm$ 0.88 & \textbf{0.13 $\pm$ 0.11} & \textbf{0.70 $\pm$ 0.33} & \textbf{0.22 $\pm$ 0.12} & \textbf{0.59 $\pm$ 0.24} & \textbf{352.38} \\
\midrule
Node2vec & 72.71 $\pm$ 0.46 & 61.71 $\pm$ 0.32 & 0.36 $\pm$ 0.05 & 2.06 $\pm$ 0.51 & 0.44 $\pm$ 0.09 & 1.20 $\pm$ 0.49 & 798.01 \\
Fairwalk (Region) & 72.78 $\pm$ 0.44 & 61.99 $\pm$ 0.40 & 0.32 $\pm$ 0.09 & 1.61 $\pm$ 0.64 & 0.40 $\pm$ 0.15 & 1.61 $\pm$ 0.49 & 829.41 \\
Fairwalk (Gender) & 72.63 $\pm$ 0.49 & 61.79 $\pm$ 0.46 & 0.34 $\pm$ 0.12 & 1.50 $\pm$ 0.67 & 0.42 $\pm$ 0.13 & 0.98 $\pm$ 0.47 & 784.69 \\
\methodname{}-Node2vec & 70.10 $\pm$ 0.61 & 57.77 $\pm$ 0.92 & \textbf{0.23 $\pm$ 0.06} & \textbf{1.22 $\pm$ 0.38} & \textbf{0.14 $\pm$ 0.08} & \textbf{0.87 $\pm$ 0.4}8 & \textbf{230.85} \\
\midrule
CFGE$_{\lambda=1}$ & 71.09 $\pm$ 0.76 & 60.21 $\pm$ 1.03 & 0.28 $\pm$ 0.18 & 0.96 $\pm$ 0.49 & 0.26 $\pm$ 0.14 & 0.68 $\pm$ 0.28 & 9572.75 \\
CFGE$_{\lambda=10}$ & 68.93 $\pm$ 0.95 & 58.98 $\pm$ 0.59 & 0.26 $\pm$ 0.19 & 1.11 $\pm$ 0.59 & 0.31 $\pm$ 0.13 & 0.50 $\pm$ 0.51 & 9447.50 \\
CFGE$_{\lambda=50}$ & 58.49 $\pm$ 1.25 & 52.08 $\pm$ 0.77 & 0.17 $\pm$ 0.10 & 0.54 $\pm$ 0.31 & 0.29 $\pm$ 0.21 & 0.40 $\pm$ 0.28 & 9521.35 \\
CFGE$_{\lambda=100}$ & 56.01 $\pm$ 0.59 & 51.04 $\pm$ 0.31 & 0.16 $\pm$ 0.06 & 0.29 $\pm$ 0.25 & 0.14 $\pm$ 0.06 & 0.23 $\pm$ 0.13 & 9391.33 \\
% CFGE$_{\lambda=1000}$ & 50.66 $\pm$ 0.77 & 51.22 $\pm$ 0.02 & 0.00 $\pm$ 0.01 & 0.00 $\pm$ 0.00 & 0.00 $\pm$ 0.00 & 0.00 $\pm$ 0.00 & 9569.36 \\
\bottomrule
\end{tabular}
\end{small}
\end{center}
\vskip -0.12in
\end{table*}

\subsection{Fairness towards Multiple Sensitive Attributes}
\label{sec:case-study}
To explore how \methodname{} learns fair representations towards multiple sensitive attributes, we conduct an experiment on the Pokec-n dataset. Pokec-n has a multi-class predicted label and two sensitive attributes. Most baselines cannot process such datasets since they restrictively only cater to the case of a single binary label or sensitive attribute. 
CFGE~\cite{bose2019compositional} is the only baseline that accommodates multiple sensitive attributes. 
Therefore we compare \methodname{} with CFGE~\cite{bose2019compositional} and graph embedding methods, and we set $c=4$ for our work. Note that FairWalk can consider only one sensitive attribute at a time, thus we run it with each one of the two sensitive attributes and show the results of both runs. In \autoref{table:pokec}, we show the results of all methods on Pokec-n in terms of utility, efficiency, and fairness with respect to two sensitive attributes. First of all, with respect to standard embedding methods, \methodname{} improves the efficiency and the fairness towards both sensitive attributes while the utility remains competitive. For example, \methodname{} reduces the $\Delta_{DP}$ of DeepWalk by $65.8\%$ on `region' and $51.1\%$ on `gender', respectively. Second, although FairWalk also fulfills the fairness towards the assigned attribute, \methodname{}-Node2vec has a better fairness score on both attributes. 
Third, \methodname{} significantly outperforms CFGE on efficiency given that they have comparable utility and fairness performance. To study the performance of CFGE, we tune the hyperparameter $\lambda$ which controls the strength of fairness. We pick $\lambda$ from 1 to 100. Going beyond 100 we find the drop in utility exceeds 20\% which is often unacceptable.  When the constraint is strict ($\lambda = 100$), CFGE has better fairness outcomes at a significant cost to the utility. For $\lambda = 1, 10$ and $50$, CFGE and \methodname{} have competitive performance in terms of fairness and utility tradeoff.  However, while CFGE takes around 9500 seconds, \methodname{} finishes in only 200-300 seconds, which is up to $46 \times$ faster.

%%%%
% Ablation study
%%%%
\subsection{Ablation Study}
\label{sec:ablation}
In the ablation study, we showcase the impact of coarsen level $c$ on \methodname{}'s performance and the effectiveness of its modules in the fairness of learned representations. In our supplement, % Appendix \ref{sec:full_tradeoff}
we also play with hyperparameters $\lambda_{c}$ and $\lambda_{r}$ to study the trade-off between fairness and effectiveness.
% Due to the space limit, we only include the results with NetMF on Credit. Other results are similar and can be found in \autoref{sec:full_ablation_study}.

% In the ablation study, we observe the impact of the key parameter coarsen level $c$ on \methodname{}'s performance and the contribution of its modules to the fairness of learned representations.

\subsubsection{Impact of coarsening} We vary the coarsen level $c$ to observe its impact on graph sizes and model performance.
\autoref{table:ablation_levels} shows the results with NetMF on the Credit dataset. Other results are similar and can be found in our supplement. Specifically, we study:

\noindent \textbf{P1: How the graph changes after each time of coarsening.} We observe that increasing $c$ exponentially reduces the numbers of nodes and edges, which corroborates the analysis in Section \ref{sec:analysis}.

\noindent \textbf{P2: How \methodname{} improves the efficiency by coarsening.} Generally, the efficiency is significantly improved when $c$ increases. A small $c$ (e.g., $c=1$) may make \methodname{} slower because the time of coarsening and refinement outweighs the saved time of learning embeddings when the coarsened graph is not small enough.

\noindent \textbf{P3: How the fairness evolves with varying $c$.} In terms of fairness in the downstream task, we observe that increasing $c$ can visibly improve the fairness of representations. For example, vanilla NetMF has $\Delta_{DP}=2.66$ and $\Delta_{EO}=1.34$, which is improved to $\Delta_{DP}=0.68$ and $\Delta_{EO}=0.68$ by \methodname{} ($c=4$).

\noindent \textbf{P4: How the utility is impacted by the information loss.} We find increasing $c$ leads to a slight decrease in AUROC and F1 scores. The AUROC score only decreases by $0.3\%$ after \methodname{} coarsens the graph 4 times. In some cases, \methodname{} achieves a better utility than the base embedding method (i.e., with $c=1$). Given the little cost of utility, we suggest using a large $c$ for the sake of fairness and efficiency.

\begin{table*}[t]
\caption{Impact of coarsen level on graph sizes and \methodname{}'s performance.} % (variances omitted for space)
\label{table:ablation_levels}
\vspace{-0.1in}
\begin{center}
\begin{small}
% Average KL divergence of each vector vs the overall distribution
% 0.3013795018196106
% 2022-08-03 13:38:23 INFO     Level 1 --- # nodes: 15033 , # edges: 1099047, diversity: 0.24875056743621826
% 2022-08-03 13:38:23 INFO     Level 2 --- # nodes: 7544 , # edges: 511754, diversity: 0.21370656788349152
% 2022-08-03 13:38:23 INFO     Level 3 --- # nodes: 3789 , # edges: 276229, diversity: 0.15492838621139526
% 2022-08-03 13:38:23 INFO     Level 4 --- # nodes: 1899 , # edges: 166675, diversity: 0.09292483329772949

% Average variance of each dimension
%  0.28550395369529724
% 2022-08-03 14:17:20 INFO     Level 1 --- # nodes: 15033 , # edges: 1099047, diversity: 0.25003382563591003
% 2022-08-03 14:17:20 INFO     Level 2 --- # nodes: 7544 , # edges: 511754, diversity: 0.22849367558956146
% 2022-08-03 14:17:20 INFO     Level 3 --- # nodes: 3789 , # edges: 276229, diversity: 0.18693852424621582
% 2022-08-03 14:17:20 INFO     Level 4 --- # nodes: 1899 , # edges: 166675, diversity: 0.13605210185050964
% \resizebox{0.92\linewidth}{!}{
\begin{tabular}{r|r|rrrr}
\toprule
 % \makecell[c]{\multirow{2}{*}{Metric}} & \makecell[c]{Vanilla} &  \multicolumn{4}{c}{Coarsen level $c$} \\
 % & \makecell[c]{NetMF} & 1 & 2 & 3 & 4 \\
 \makecell[c]{Metric} & \makecell[c]{Vanilla NetMF} & $c=1$ & $c=2$ & $c=3$ & $c=4$ \\
\midrule
 \# Nodes & 30000 & 15033 & 7544 & 3789 & 1899 \\
 \# Edges & 1.44M & 550K & 256K & 138K & 83K \\
%  \# Edges & 1.44M & 549.52K & 255.88K & 138.11K & 83.34K \\
\midrule
%  Time $(\downarrow)$ & 240.38 & 487.60 & 208.75 & 139.79 & \textbf{90.28} \\
% \midrule
%  % $\vartheta~(\downarrow)$ & 0.30 & 0.25 & 0.21 & 0.15 & 0.09 \\
%  $\Delta_{DP}~(\downarrow)$ & 2.66 & 2.23 & 2.07 & 1.88 & \textbf{0.68} \\
%  $\Delta_{EO}~(\downarrow)$ & 1.34 & 1.22 & 1.14 & 1.09 & \textbf{0.68} \\

Time $(\downarrow)$ & 240.38 & 487.60 & 208.75 & 139.79 & \textbf{90.28} \\
\midrule
 % $\vartheta~(\downarrow)$ & 0.30 & 0.25 & 0.21 & 0.15 & 0.09 \\
 $\Delta_{DP}~(\downarrow)$ & 2.66 $\pm$ 0.55 & 2.23 $\pm$ 0.38 & 2.07 $\pm$ 0.22 & 1.88 $\pm$ 0.37 & \textbf{0.68 $\pm$ 0.50} \\
 $\Delta_{EO}~(\downarrow)$ & 1.34 $\pm$ 0.93 & 1.22 $\pm$ 0.80 & 1.14 $\pm$ 0.76 & 1.09 $\pm$ 0.69 & \textbf{0.68 $\pm$ 0.66} \\
\midrule
 AUROC ($\uparrow$) & 74.93 $\pm$ 0.43 & \textbf{75.12 $\pm$ 0.52} & 74.95 $\pm$ 0.40 & 74.80 $\pm$ 0.41 & 74.69 $\pm$ 0.43 \\
 F1 ($\uparrow$) & 88.36 $\pm$ 0.08 & 88.34 $\pm$ 0.11 & \textbf{88.41 $\pm$ 0.11} & 88.34 $\pm$ 0.11 & 88.31 $\pm$ 0.08 \\
\bottomrule
\end{tabular}
% }
\end{small}
\end{center}
\vspace{-0.05in}
\end{table*}

%%%%
% Effectiveness of each module
%%%%

\begin{table*}[!t]
\caption{Ablation study of each module's effectiveness in fairness (with NetMF).}
\label{table:ablation_modules_main}
\vspace{-0.1in}
\begin{center}
\begin{small}
% \resizebox{\linewidth}{!}{
\begin{tabular}{c|l|rr|rr|r}
\toprule
Dataset & \makecell[c]{Method} & \makecell[c]{AUROC ($\uparrow$)} & \makecell[c]{F1 ($\uparrow$)} & \makecell[c]{$\Delta_{DP}~(\downarrow)$} & \makecell[c]{$\Delta_{EO}~(\downarrow)$} & \makecell[c]{Time $(\downarrow)$}\\
\midrule
% \multirow{4}{*}{German} & FairMILE & 61.93 $\pm$ 3.38 & 82.35 $\pm$ 0.00 & \textbf{0.00 $\pm$ 0.00} & \textbf{0.00 $\pm$ 0.00} & 6.31 \\
%  & FairMILE w/o fair coarsening & 63.61 $\pm$ 2.91 & 82.27 $\pm$ 0.58 & 2.20 $\pm$ 0.98 & 2.94 $\pm$ 3.09 & 5.98 \\
%  & FairMILE w/o fair refinement & 63.37 $\pm$ 2.91 & 81.08 $\pm$ 0.60 & 2.94 $\pm$ 1.58 & 4.04 $\pm$ 2.13 & 5.64 \\
%  & MILE & 63.02 $\pm$ 2.76 & 82.04 $\pm$ 0.76 & 3.28 $\pm$ 2.10 & 4.02 $\pm$ 1.95 & 5.06 \\
% \hline
% \multirow{4}{*}{Recidivism} & FairMILE & 89.52 $\pm$ 0.50 & 77.65 $\pm$ 0.47 & \textbf{2.81 $\pm$ 0.50} & \textbf{1.51 $\pm$ 1.10} & 29.66 \\
%  & FairMILE w/o fair coarsening & 90.23 $\pm$ 0.43 & 79.62 $\pm$ 0.64 & 3.18 $\pm$ 0.33 & 2.68 $\pm$ 1.22 & 22.36 \\
%  & FairMILE w/o fair refinement & 90.25 $\pm$ 0.23 & 79.94 $\pm$ 0.50 & 3.27 $\pm$ 0.36 & 2.60 $\pm$ 1.33 & 26.01 \\
%  & MILE & 90.75 $\pm$ 0.30 & 80.46 $\pm$ 0.64 & 3.29 $\pm$ 0.32 & 2.86 $\pm$ 1.30 & 18.42 \\
% \hline
% \multirow{4}{*}{Credit} & FairMILE & 74.69 $\pm$ 0.43 & 88.31 $\pm$ 0.08 & \textbf{0.68 $\pm$ 0.50} & \textbf{1.36 $\pm$ 1.31} & 90.28 \\
%  & FairMILE w/o fair coarsening & 74.65 $\pm$ 0.34 & 88.30 $\pm$ 0.12 & 1.50 $\pm$ 0.52 & 1.84 $\pm$ 1.48 & 57.21 \\
%  & FairMILE w/o fair refinement & 74.42 $\pm$ 0.33 & 88.25 $\pm$ 0.15 & 3.15 $\pm$ 0.49 & 3.42 $\pm$ 1.37 & 82.01 \\
%  & MILE & 74.65 $\pm$ 0.41 & 88.33 $\pm$ 0.12 & 2.54 $\pm$ 0.48 & 2.75 $\pm$ 1.62 & 49.98 \\
% \hline
\multirow{4}{*}{German} & FairMILE & 61.93 $\pm$ 3.38 & 82.35 $\pm$ 0.00 & \textbf{0.00 $\pm$ 0.00} & \textbf{0.00 $\pm$ 0.00} & 6.31 \\
 & FairMILE w/o fair coarsening & 63.61 $\pm$ 2.91 & 82.27 $\pm$ 0.58 & 2.20 $\pm$ 0.98 & 1.47 $\pm$ 1.54 & 5.98 \\
 & FairMILE w/o fair refinement & 63.37 $\pm$ 2.91 & 81.08 $\pm$ 0.60 & 2.94 $\pm$ 1.58 & 2.02 $\pm$ 1.06 & 5.64 \\
 & MILE & 63.02 $\pm$ 2.76 & 82.04 $\pm$ 0.76 & 3.28 $\pm$ 2.10 & 2.01 $\pm$ 0.97 & 5.06 \\
\midrule
\multirow{4}{*}{Recidivism} & FairMILE & 89.52 $\pm$ 0.50 & 77.65 $\pm$ 0.47 & \textbf{2.81 $\pm$ 0.50} & \textbf{0.75 $\pm$ 0.55} & 29.66 \\
 & FairMILE w/o fair coarsening & 90.23 $\pm$ 0.43 & 79.62 $\pm$ 0.64 & 3.18 $\pm$ 0.33 & 1.34 $\pm$ 0.61 & 22.36 \\
 & FairMILE w/o fair refinement & 90.25 $\pm$ 0.23 & 79.94 $\pm$ 0.50 & 3.27 $\pm$ 0.36 & 1.30 $\pm$ 0.66 & 26.01 \\
 & MILE & 90.75 $\pm$ 0.30 & 80.46 $\pm$ 0.64 & 3.29 $\pm$ 0.32 & 1.43 $\pm$ 0.65 & 18.42 \\
\midrule
\multirow{4}{*}{Credit} & FairMILE & 74.69 $\pm$ 0.43 & 88.31 $\pm$ 0.08 & \textbf{0.68 $\pm$ 0.50} & \textbf{0.68 $\pm$ 0.66} & 90.28 \\
 & FairMILE w/o fair coarsening & 74.65 $\pm$ 0.34 & 88.30 $\pm$ 0.12 & 1.50 $\pm$ 0.52 & 0.92 $\pm$ 0.74 & 57.21 \\
 & FairMILE w/o fair refinement & 74.42 $\pm$ 0.33 & 88.25 $\pm$ 0.15 & 3.15 $\pm$ 0.49 & 1.71 $\pm$ 0.68 & 82.01 \\
 & MILE & 74.65 $\pm$ 0.41 & 88.33 $\pm$ 0.12 & 2.54 $\pm$ 0.48 & 1.38 $\pm$ 0.81 & 49.98 \\
\bottomrule
\end{tabular}
% }
\end{small}
\end{center}
\vskip -0.05in
\end{table*}

% \begin{table}[!t]
% \caption{Ablation study of each module's effectiveness in fairness (FC denotes Fair Coarsening, and FR denotes Fair Refinement; variances omitted for space)}
% \label{table:ablation_modules_main}
% \vskip 0.1in
% \begin{center}
% \begin{small}
% \begin{tabular}{r|rrrr}
% \toprule
%  \makecell[c]{\multirow{2}{*}{Metric}} & \makecell[c]{\multirow{2}{*}{FairMILE}} & \makecell[c]{FairMILE} & \makecell[c]{FairMILE} & \makecell[c]{\multirow{2}{*}{MILE}} \\
%  & & \makecell[c]{w/o FC}& \makecell[c]{w/o FR} & \\
% \midrule
%  AUROC ($\uparrow$) & 74.69 & 74.65 & 74.42 & 74.65 \\
%  F1 ($\uparrow$) & 88.31 & 88.30 & 88.25 & 88.33 \\
%  $\Delta_{DP}~(\downarrow)$ & \textbf{0.68} & 1.50 & 3.15 & 2.54 \\
%  $\Delta_{EO}~(\downarrow)$ & \textbf{0.68} & 0.92 & 1.71 & 1.38 \\
%  Time $(\downarrow)$ & 90.28 & 57.21 & 82.01 & 49.98 \\
% %  AUROC ($\uparrow$) & 74.69 $\pm$ 0.43 & 74.65 $\pm$ 0.34 & 74.42 $\pm$ 0.33 & 74.65 $\pm$ 0.41 \\
% %  F1 ($\uparrow$) & 88.31 $\pm$ 0.08 & 88.30 $\pm$ 0.12 & 88.25 $\pm$ 0.15 & 88.33 $\pm$ 0.12 \\
% %  $\Delta_{DP}~(\downarrow)$ & \textbf{0.68 $\pm$ 0.50} & 1.50 $\pm$ 0.52 & 3.15 $\pm$ 0.49 & 2.54 $\pm$ 0.48 \\
% %  $\Delta_{EO}~(\downarrow)$ & \textbf{0.68 $\pm$ 0.66} & 0.92 $\pm$ 0.74 & 1.71 $\pm$ 0.68 & 1.38 $\pm$ 0.81 \\
% %  Time $(\downarrow)$ & 90.28 & 57.21 & 82.01 & 49.98 \\
% \bottomrule
% \end{tabular}
% \end{small}
% \end{center}
% \vskip -0.1in
% \end{table}
% % \vspace{-4mm}

\subsubsection{Effectiveness of fairness-aware modules} To evaluate the effectiveness of our fairness-aware modules for graph coarsening and refinement,
% we observe the change in performance when we replace one or both of them. In this study, we choose MILE~\cite{liang2021mile}, a multi-level framework without fairness considerations, as the baseline. Additionally, we replace the coarsening or refinement module in \methodname{} with its counterpart in MILE. 
we observe the change in performance when we remove the fairness-aware designs. Additionally, we choose MILE~\cite{liang2021mile} as a baseline of scalable embedding methods without fairness considerations.
\autoref{table:ablation_modules_main} reports the results on three datasets.
% with NetMF used as the base embedding method. 
% Similar results are observed on other datasets (see Appendix \ref{sec:full_effectiveness}). % with other embedding methods. 
The result shows that \methodname{} outperforms MILE in terms of fairness, which demonstrates our effectiveness in fairness improvement. We also notice that \textbf{the fairness scores decline on all datasets when removing either fairness-aware design in \methodname{}}, indicating that these designs effectively mitigate the bias in the learned embeddings. Furthermore, the little differences between the utility scores of these methods demonstrate that \methodname{} is able to improve fairness without impacting utility when compared with MILE. However, this improved fairness does come at some cost to efficiency w.r.t MILE (which is always faster than \methodname{}).

%%%%
% Comparison in Link Prediction
%%%%

\begin{table*}[!t]
\caption{Comparison on utility, fairness, and efficiency metrics in link prediction between \methodname{} and other baselines.}
\label{table:link_prediction_short}
\vspace{-0.1in}
\begin{center}
\begin{small} 
% \resizebox{0.88\linewidth}{!}{
\begin{tabular}{c|l|lll|rr|r}
\toprule
Dataset & \makecell[c]{Method} & \makecell[c]{AUROC ($\uparrow$)} & \makecell[c]{AP ($\uparrow$)} & \makecell[c]{Accuracy ($\uparrow$)} & \makecell[c]{$\Delta_{DP, ~\mathrm{LP}}~(\downarrow)$} & \makecell[c]{$\Delta_{EO, ~\mathrm{LP}}~(\downarrow)$} & \makecell[c]{Time $(\downarrow)$}\\
\midrule
\multirow{11}{*}{Cora}
 & VGAE & \textbf{90.86 $\pm$ 0.79} & \textbf{92.81 $\pm$ 0.89} & \textbf{82.68 $\pm$ 1.19} & 47.51 $\pm$ 2.47 & 24.12 $\pm$ 3.29 & \textbf{15.17} \\
%  & EDITS & & & & & & \\
 & FairAdj$_{\mathrm{T2=2}}$ & 89.56 $\pm$ 1.06 & 91.31 $\pm$ 1.21 & 81.57 $\pm$ 1.22 & 45.47 $\pm$ 2.52 & 20.44 $\pm$ 3.48 & 56.76 \\
 & FairAdj$_{\mathrm{T2=5}}$ & 88.49 $\pm$ 1.30 & 90.34 $\pm$ 1.36 & 80.97 $\pm$ 0.83 & 42.39 $\pm$ 2.95 & 16.75 $\pm$ 3.26 & 70.37 \\
 & FairAdj$_{\mathrm{T2=20}}$ & 85.97 $\pm$ 0.62 & 87.84 $\pm$ 0.53 & 77.68 $\pm$ 0.54 & 35.67 $\pm$ 1.45 & \textbf{11.26 $\pm$ 3.07} & 145.01 \\
 & CFGE & 90.76 $\pm$ 1.27 & 92.59 $\pm$ 1.30 & 82.35 $\pm$ 1.26 & \textbf{32.06 $\pm$ 1.33} & 13.73 $\pm$ 1.63 & 531.79 \\
 \cmidrule{2-8}
 & NetMF & 91.33 $\pm$ 1.35 & 92.96 $\pm$ 1.16 & \textbf{86.41 $\pm$ 1.10} & 42.03 $\pm$ 2.57 & 17.24 $\pm$ 2.48 & 22.80 \\
 & \methodname{}-NetMF & \textbf{92.51 $\pm$ 0.74} & \textbf{93.56 $\pm$ 1.03} & 84.42 $\pm$ 0.87 & \textbf{41.19 $\pm$ 1.47} & \textbf{16.09 $\pm$ 1.46} & \textbf{8.16} \\
 \cmidrule{2-8}
 & DeepWalk & 88.52 $\pm$ 1.01 & 89.84 $\pm$ 0.77 & 77.78 $\pm$ 0.99 & 42.53 $\pm$ 1.55 & 21.58 $\pm$ 2.19 & 47.30 \\
 & \methodname{}-DeepWalk & \textbf{92.00 $\pm$ 1.27} & \textbf{93.67 $\pm$ 0.93} & \textbf{82.26 $\pm$ 0.64} & \textbf{33.17 $\pm$ 1.02} & \textbf{14.31 $\pm$ 1.40} & \textbf{19.40} \\
 \cmidrule{2-8}
 & Node2vec & 90.87 $\pm$ 1.03 & 92.03 $\pm$ 0.77 & 80.15 $\pm$ 0.55 & 45.62 $\pm$ 1.80 & 23.37 $\pm$ 3.09 & 21.30 \\
 & Fairwalk & 89.49 $\pm$ 1.24 & 91.00 $\pm$ 0.88 & 78.92 $\pm$ 0.99 & 42.21 $\pm$ 2.28 & 18.68 $\pm$ 2.68 & 22.71 \\
 & \methodname{}-Node2vec & \textbf{91.77 $\pm$ 0.89} & \textbf{93.13 $\pm$ 1.03} & \textbf{85.16 $\pm$ 0.74} & \textbf{28.99 $\pm$ 0.98} & \textbf{11.65 $\pm$ 0.81} & \textbf{13.71} \\
\midrule
\multirow{11}{*}{Citeseer}
 & VGAE & \textbf{88.24 $\pm$ 1.83} & \textbf{90.53 $\pm$ 2.78} & \textbf{81.67 $\pm$ 3.39} & 24.82 $\pm$ 2.67 & 7.13 $\pm$ 3.32 & \textbf{20.33} \\
%  & EDITS & & & & & & \\
 & FairAdj$_{\mathrm{T2=2}}$ & 87.90 $\pm$ 1.81 & 89.99 $\pm$ 2.81 & 81.61 $\pm$ 3.02 & 24.17 $\pm$ 2.65 & 6.49 $\pm$ 3.34 & 84.10 \\
 & FairAdj$_{\mathrm{T2=5}}$ & 87.47 $\pm$ 1.74 & 89.48 $\pm$ 2.67 & 81.50 $\pm$ 3.01 & 23.65 $\pm$ 2.72 & 6.47 $\pm$ 3.47 & 109.07 \\
 & FairAdj$_{\mathrm{T2=20}}$ & 86.90 $\pm$ 2.23 & 88.61 $\pm$ 3.23 & 80.41 $\pm$ 3.35 & 22.74 $\pm$ 3.23 & 6.47 $\pm$ 3.38 & 240.16 \\
 & CFGE & 87.36 $\pm$ 3.29 & 90.45 $\pm$ 3.11 & 71.12 $\pm$ 13.42 & \textbf{13.60 $\pm$ 4.86} & \textbf{3.00 $\pm$ 1.39} & 478.18 \\
 \cmidrule{2-8}
 & NetMF & 87.62 $\pm$ 0.73 & 91.20 $\pm$ 0.46 & 84.53 $\pm$ 0.30 & 23.32 $\pm$ 1.48 & 4.66 $\pm$ 1.72 & 19.35 \\
 & \methodname{}-NetMF & \textbf{89.22 $\pm$ 0.34} & \textbf{91.71 $\pm$ 0.69} & \textbf{85.30 $\pm$ 0.58} & \textbf{20.96 $\pm$ 1.65} & \textbf{2.79 $\pm$ 1.45} & \textbf{7.57} \\
 \cmidrule{2-8}
 & DeepWalk & 88.83 $\pm$ 0.26 & 90.96 $\pm$ 0.31 & 81.59 $\pm$ 1.54 & 23.99 $\pm$ 1.88 & 5.55 $\pm$ 3.14 & 48.10 \\
 & \methodname{}-DeepWalk & \textbf{89.50 $\pm$ 1.02} & \textbf{92.55 $\pm$ 0.70} & \textbf{86.39 $\pm$ 0.37} & \textbf{17.67 $\pm$ 1.18} & \textbf{2.62 $\pm$ 1.31} & \textbf{21.20} \\
 \cmidrule{2-8}
 & Node2vec & 89.00 $\pm$ 0.65 & 91.77 $\pm$ 0.35 & 83.88 $\pm$ 0.81 & 22.85 $\pm$ 1.57 & 3.52 $\pm$ 2.17 & 16.88 \\
 & Fairwalk & 88.86 $\pm$ 0.85 & 91.71 $\pm$ 0.33 & 83.52 $\pm$ 0.61 & 23.59 $\pm$ 1.38 & 4.13 $\pm$ 2.43 & 16.78 \\
 & \methodname{}-Node2vec & \textbf{89.54 $\pm$ 0.95} & \textbf{92.31 $\pm$ 0.52} & \textbf{86.91 $\pm$ 0.78} & \textbf{12.49 $\pm$ 0.91} & \textbf{2.08 $\pm$ 1.08} & \textbf{11.72} \\
\bottomrule
\end{tabular}
% }
\end{small}
\end{center}
% \vspace{-.15in}
\end{table*}

\subsection{Results for Link Prediction}
\label{sec:link_prediction_short}
We next evaluate \methodname{} in the context of link prediction. \autoref{table:link_prediction_short} only shows the results on Cora and Citeseer.
% due to lack of space.
Full results and detailed analysis are in our supplement.
% \autoref{sec:full_link_prediction}. 
First, \methodname{} effectively mitigates the bias in link prediction. Our framework has lower $\Delta_{DP, ~\mathrm{LP}}$ and $\Delta_{EO, ~\mathrm{LP}}$ than the specialized approaches and graph embedding approaches. On the other hand, \methodname{} is more efficient than other baselines, especially the fairness-aware competitors FairAdj and CFGE. Finally, \methodname{} has a similar or slightly better utility compared with other approaches. This demonstrates that \methodname{} consistently achieves the tradeoff between fairness, efficiency, and utility in a different downstream task.

%% file: sections/conclusion.tex
In this paper, we study the problem of fair graph representation learning and propose a general framework \methodname{}, which can incorporate fairness considerations with any unsupervised graph embedding algorithms and learn fair embeddings towards one or multiple sensitive attributes. 
% \methodname{} also improves the efficiency of learning fair representations by leveraging the multi-level framework. 
% In addition, it is agnostic to the downstream tasks so the learned representations can be used for distinct purposes. 
We conduct comprehensive experiments to demonstrate that with respect to state-of-the-art techniques for fair graph representation learning, our work achieves similar or better performance in terms of utility and fairness, while \methodname{} can significantly outperform them on the axis of efficiency (up to two orders of magnitude faster). Planned future work includes evaluating the use of \methodname{} for real-world graph-based model auditing \cite{maneriker2023online} in deployed online settings. 
% We are currently working on extending \methodname{} to scale to much larger datasets building on lessons learned from the high-performance MILE ecosystem\cite{liang2021mile,he2021distmile}.

%% file: sections/techreport.tex
\section{Full Theoretical Analysis}
\label{sec:time_analysis}

% \subsection{Time Complexity}
\begin{corollary}
\label{crl:coarse_results}
Algorithm \ref{alg:coarsen} coarsens a graph $\mathcal{G}_{i} = (\mathcal{V}_{i}, \mathcal{E}_{i})$ into a smaller graph $\mathcal{G}_{i+1} = (\mathcal{V}_{i+1}, \mathcal{E}_{i+1})$ such that $\frac{1}{2}|\mathcal{V}_{i}| \leq |\mathcal{V}_{i+1}| \leq |\mathcal{V}_{i}|$ and $|\mathcal{E}_{i+1}| \leq |\mathcal{E}_{i}|$.
\end{corollary}
\begin{proof}
In the optimal case, all nodes in $\mathcal{G}_{i}$ get matched and therefore $|\mathcal{V}_{i+1}| = \frac{1}{2} |\mathcal{V}_{i}|$. The worst case is that all nodes are isolated so the number of nodes does not decrease.

Our coarsening algorithm adds an edge $(u, v)$ in $\mathcal{G}_{i+1}$ if and only if there exists at least one edge in $\mathcal{G}_{i}$ that connects one of $u$'s child nodes and one of $v$'s child nodes. Therefore $|\mathcal{E}_{i+1}| \leq |\mathcal{E}_{i}|$.
\end{proof}

\begin{lemma}
\label{lem:time_coarsen}
In the phase of graph coarsening, \methodname{} consumes $O\left( M\left( \sum_{i=0}^{c-1} (|\mathcal{V}_{i}| + |\mathcal{E}_{i}|) \right) \right)$ time. % and $O\left (M|\mathcal{V}_{0}| + |\mathcal{E}_{0}| \right )$ memory.
\end{lemma}

\begin{proof}
Without loss of generality, we assume the input graph is $G_{i}$ ($i < c$). For each edge, Algorithm \ref{alg:coarsen} computes the fairness-aware edge weight in $O(M)$ time. Hence it takes $O\left (M (|\mathcal{E}_{i}| + |\mathcal{V}_{i}|)\right )$ time to match and merge nodes. The time of creating a coarsened graph after node matching is also $O\left (M (|\mathcal{E}_{i}| + |\mathcal{V}_{i}|)\right )$, which is mainly used for computing the attribute distribution of new nodes and the weights of new edges in the coarsened graph. Summing up the $c$ coarsen levels, the total time complexity of graph coarsening is $O\left( M\left( \sum_{i=0}^{c-1} (|\mathcal{V}_{i}| + |\mathcal{E}_{i}|) \right) \right)$.
\end{proof}

\begin{corollary}
\label{crl:time_baseembed}
\methodname{} can reduce the time of graph embedding exponentially in the optimal case since $|\mathcal{V}_{c}| \geq 2^{-c} |\mathcal{V}_{0}|$.
\end{corollary}

\begin{proof}
In Corollary \ref{crl:coarse_results}, we analyze that the number of nodes can be reduced by up to half at each coarsen level. Thus the time complexity of base embedding can also be reduced exponentially when $c$ increases.
\end{proof}

\begin{lemma}
If the refinement model has $l$ layers, the time complexity of refinement is $O \left ( l(d+M) \left [ \sum_{i=0}^{c-1} \left ( |\mathcal{E}_{i}| + d |\mathcal{V}_{i}| \right ) \right ] \right )$. % The model uses $O(|\mathcal{V}_{0}|(d+M))$ memory. 
\end{lemma}

\begin{proof}
We again assume the input graph is $\mathcal{G}_{i+1}$ ($i > 0$) without loss of generality. Before applying the model, \methodname{} projects the embeddings from the supernodes to the child nodes of $\mathcal{G}_{i}$ in $O(d|\mathcal{V}_{i}|)$ time. In each layer of the model, \methodname{} needs $O((d + M)|\mathcal{V}_{i}|)$ to concatenate the input with the sensitive attributes. The following message passing process and the matrix multiplication take $O((d+M)|\mathcal{E}_{i}|)$ and $O(d(d+M)|\mathcal{V}_{i}|)$ time, respectively. Therefore the time complexity of each layer is $O\left((d+M) (|\mathcal{E}_{i}| + d|\mathcal{V}_{i}|) \right)$. Finally, the total time complexity of applying the refinement model is $O \left ( l(d+M) \left [ \sum_{i=0}^{c-1} \left ( |\mathcal{E}_{i}| + d |\mathcal{V}_{i}| \right ) \right ] \right )$.
\end{proof}

\noindent \textbf{Theorem \ref{thm:fairness_main}}. When $L_f$ is minimized, the 2-norm of the difference between the mean embeddings of any two demographic groups regarding a given sensitive attribute is bounded by
\begin{equation}
    \left \| \bm{\mu}_{p} - \bm{\mu}_{q} \right \|_{2} \leq 2 (1 - \min(\beta_{p}, \beta_{q}))
\label{eqn:fairness_theorem}
\end{equation}
where $p, q$ are any two different values of the given sensitive attribute (e.g., gender or race). For $i \in \{p, q\}$, $\bm{\mu}_{i}$ denotes the mean embedding values of nodes from group $i$, and $\beta_{i}$ denotes the ratio of nodes from group $i$ that have at least one inter-group edge.
\label{thm:fairness}

\begin{proof}
    Our fairness loss $L_f$ in Equation (\ref{eqn:fairloss}) is the average negative cosine similarity of embeddings of all connected node pairs $(u, v) \in \mathcal{E}'_{c}$ with diverse sensitive attributes, which means $\phi(u, v) \geq \gamma$. To reach the global minimum of $L_f$, the parameters $\{\Theta_{i}\}_{i=1}^{l}$ are optimized to completely focus on sensitive attributes and generate identical embeddings for the nodes, i.e., the similarity is 1. In general, using a smaller $\gamma$ will have more edges impacted by the fairness objective. $\gamma=0$ is the strictest value that always leads to all nodes with inter-group connections having the same representations (denoted as $\bm{\bar{h}}$). A special case is when $c=1$, $\phi(u, v) = 1$ holds for any $u, v \in \mathcal{V}_0$ such that $s_u \neq s_v$. Therefore any $\gamma$ enforces fairness in the same strength.

    Recall that $p, q$ are any two different values of the given sensitive attribute (e.g., gender). For $i \in \{p, q\}$, let $\mathcal{U}_{i}$ be the set of nodes with attribute value $i$ (e.g., all nodes in $\mathcal{U}_{i}$ share the same gender $i$), Then let $\mathcal{U}_{ii}$ be the subset of $\mathcal{U}_{i}$ in which the nodes are only connected to nodes in $\mathcal{U}_{i}$. Note that $\beta_{i} = 1 - |\mathcal{U}_{ii}| / |\mathcal{U}_{i}|$. The learned embedding of node $u$ is denoted as $\bm{h}_{u}$. When $L_{f}$ is minimized with $\gamma=0$, we have
    \begin{equation}
    \begin{split}
        & ~\left \| \bm{\mu}_{p} - \bm{\mu}_{q} \right \|_{2} \nonumber \\
        = & ~\| \beta_{p} \bm{\bar{h}} + \frac{1 - \beta_{p}}{|\mathcal{U}_{pp}|} \sum_{u \in \mathcal{U}_{pp}} \bm{h}_{u} -~\beta_{q} \bm{\bar{h}} - \frac{1 - \beta_{q}}{|\mathcal{U}_{qq}|} \sum_{v \in \mathcal{U}_{qq}} \bm{h}_{v}\|_{2} \nonumber \\
        \leq & ~|\beta_{p} - \beta_{q}| \| \bm{\bar{h}} \|_{2} + (1 - \beta_{p}) \| \frac{1}{|\mathcal{U}_{pp}|} \sum_{u \in \mathcal{U}_{pp}} \bm{h}_{u} \|_{2} + (1 - \beta_{q}) \| \frac{1}{|\mathcal{U}_{qq}|} \sum_{v \in \mathcal{U}_{qq}} \bm{h}_{v} \|_{2} \nonumber
    \end{split}
    \end{equation}

    Note that the output embeddings of our refinement model are L2-normalized, hence we finally have
    \begin{equation}
    \begin{split}
        \left \| \bm{\mu}_{p} - \bm{\mu}_{q} \right \|_{2} & \leq |\beta_{p} - \beta_{q}| + (1 - \beta_{p}) + (1 - \beta_{q}) = 2(1 - \min(\beta_{p}, \beta_{q})) \nonumber
    \end{split}
    \end{equation}
    Therefore Equation (\ref{eqn:fairness_theorem}) holds when $L_{f}$ is minimized. 
\end{proof}

\section{Dataset Description}
\label{sec:dataset_description}

In \textit{German}~\cite{agarwal2021towards}, each node is a client in a German bank, and two nodes are linked if their attributes are similar. The task is to classify a client's credit risk as good or bad, and the sensitive attribute is the client's gender. % This dataset has been used to evaluate node classification performance in the following papers\cite{xy}. 

\textit{Recidivism}~\cite{agarwal2021towards} is a graph created from a set of bail outcomes from US state courts between 1990-2009, where nodes are defendants and edges connect two nodes if they have similar past criminal records and demographic attributes. The task is to predict if a defendant will commit a violent crime or not, while the sensitive attribute is race. % This dataset has been used to evaluate node classification performance in the following papers\cite{xy}. 

\textit{Credit}~\cite{agarwal2021towards} consists of credit card applicants with their demographic features and payment patterns. Each node is an individual, and two nodes are connected based on feature similarity. In this dataset, age is used as the sensitive attribute, and the predicted label is whether the applicant will default on an upcoming payment. % This dataset has been used to evaluate node classification performance in the following papers\cite{xy}. 

\textit{Pokec-n}~\cite{takac2012data} is collected from a Slovakia social network. We use both region and gender as the sensitive attributes, and choose each user's field of work as the predicted label. Note that Pokec-n has multiple sensitive attributes and a multi-class target, which \methodname{} can handle by design. However, existing research~\cite{dai2021say, dong2022edits, franco2022deep} has only evaluated the use of this data with one sensitive attribute at-a-time with the target label binarized - a key limitation. We discuss how \methodname{} can redress this limitation in Section~\ref{sec:case-study}.

%%%%
% Overall comparison
%%%%

\begin{table*}[!t]
\caption{Comparison in node classification between \methodname{} and other baselines on Recidivism dataset.}
\label{table:overall_results_techreport}
\vskip -0.1in
\begin{center}
\begin{small}
\begin{tabular}{c|l|ll|rr|r}
\toprule
Dataset & \makecell[c]{Method} & \makecell[c]{AUROC ($\uparrow$)} & \makecell[c]{F1 ($\uparrow$)} & \makecell[c]{$\Delta_{DP}~(\downarrow)$} & \makecell[c]{$\Delta_{EO}~(\downarrow)$} & \makecell[c]{Time $(\downarrow)$}\\
\midrule
\multirow{11}{*}{Recidivism}
 & NetMF & \textbf{94.63 $\pm$ 0.17} & \textbf{85.46 $\pm$ 0.29} & 3.41 $\pm$ 0.21 & 1.62 $\pm$ 0.78 & 141.90 \\
 & \methodname{}-NetMF & 89.52 $\pm$ 0.50 & 77.65 $\pm$ 0.47 & \textbf{2.81 $\pm$ 0.50} & \textbf{0.75 $\pm$ 0.55} & \textbf{29.66} \\
\cmidrule{2-7}
 & DeepWalk & \textbf{93.33 $\pm$ 0.35} & \textbf{83.62 $\pm$ 0.42} & 3.47 $\pm$ 0.37 & 1.28 $\pm$ 0.60 & 303.68 \\
 & \methodname{}-DeepWalk & 86.93 $\pm$ 0.74 & 73.50 $\pm$ 0.99 & \textbf{2.71 $\pm$ 0.58} & \textbf{1.08 $\pm$ 0.77} & \textbf{45.93} \\
\cmidrule{2-7}
 & Node2vec & \textbf{92.56 $\pm$ 0.26} & \textbf{83.31 $\pm$ 0.36} & 3.61 $\pm$ 0.56 & 1.57 $\pm$ 0.97 & 136.33 \\
 & Fairwalk & 92.43 $\pm$ 0.43 & 82.99 $\pm$ 0.51 & 3.32 $\pm$ 0.24 & 1.48 $\pm$ 0.66 & 133.62 \\
 & \methodname{}-Node2vec & 87.00 $\pm$ 0.50 & 71.34 $\pm$ 0.86 & \textbf{2.75 $\pm$ 0.35} & \textbf{1.15 $\pm$ 0.65} & \textbf{38.67} \\
\cmidrule{2-7}
 & Vanilla GCN & \textbf{88.16 $\pm$ 1.72} & \textbf{77.68 $\pm$ 1.63} & 3.83 $\pm$ 0.59 & 1.46 $\pm$ 0.71 & \textbf{474.57} \\
 & FairGNN & 67.26 $\pm$ 7.80 & 44.63 $\pm$ 14.87 & \textbf{0.67 $\pm$ 0.45} & 1.24 $\pm$ 0.40 & 1071.39 \\
 & NIFTY & 77.89 $\pm$ 4.21 & 64.44 $\pm$ 6.11 & 1.34 $\pm$ 1.01 & \textbf{0.63 $\pm$ 0.42} & 1651.09 \\
 & EDITS & 79.48 $\pm$ 13.26 & 69.66 $\pm$ 13.28 & 4.39 $\pm$ 2.10 & 2.52 $\pm$ 2.04 & 1311.42 \\
 & CFGE & 60.92 $\pm$ 1.88 & 25.58 $\pm$ 6.45 & 0.81 $\pm$ 0.58 & 1.45 $\pm$ 0.88 & 2498.52 \\
\bottomrule
\end{tabular}
\end{small}
\end{center}
% \vskip -0.1in
\end{table*}

The remaining three datasets (namely, \textit{Cora}~\cite{sen2008collective}, \textit{Citeseer}~\cite{sen2008collective}, and \textit{Pubmed}~\cite{namata2012query}) are citation networks widely evaluated in the graph representation learning literature~\cite{li2020dyadic, spinelli2021fairdrop, current2022fairmod}. In these data, each node denotes a paper, and each edge links two nodes if one paper cites the other. 
%Note that edges are undirected in all datasets. 
% These datasets have also been used in the literature to evaluate fair link prediction performance\cite{xy}. 
As in prior work~\cite{li2020dyadic, spinelli2021fairdrop, current2022fairmod}, we treat the category of a paper as its sensitive attribute. The task is to predict whether a paper is cited by the other (or vice versa).

\section{Node Classification on Recidivism}
\label{sec:full_result_node_classification}

We conduct the experiments of node classification on another dataset Recidivism and revisit the questions in Section \ref{sec:exp_nc}. Results are shown in \autoref{table:overall_results_techreport}.

\textbf{A1) Fairness:} \methodname{} improves the fairness of all unsupervised graph embedding approaches. In Recidivism, \methodname{} decreases the $\Delta_{DP}$ scores of NetMF and DeepWalk by $17.6\%$ and $21.9\%$, respectively. In terms of $\Delta_{EO}$, \methodname{} improves the fairness of NetMF and DeepWalk by $53.7\%$ and $15.6\%$. On top of Node2vec, \methodname{} outperforms FairWalk in terms of both $\Delta_{DP}$ and $\Delta_{EO}$. Among the specialized methods, FairGNN has the lowest $\Delta_{DP}$ score and NIFTY has the best $\Delta_{EO}$ score which is slightly better than \methodname{}-NetMF ($0.63\%$ v.s. $0.75\%$). However, this is because these models trade too much utility for fairness (For example, in terms of AUROC, NIFTY $77.89\%$ v.s. \methodname{}-NetMF $89.52\%$). 

\textbf{A2) Efficiency:} \methodname{} is more efficient than other baselines. While GNN-based approaches take up to $2498.5$ seconds, \methodname{} on top of NetMF finishes in only $29.7$ seconds, which is $84.2 \times$ faster. Compared with the unsupervised graph embedding approaches, \methodname{} still improves the efficiency of graph embedding. 

\textbf{A3) Utility:} Compared with the base embedding methods, the utility scores of \methodname{} slightly drop which is fairly remarkable given that \methodname{} significantly improves fairness and efficiency. Among the specialized approaches, all approaches except the vanilla GCN are outperformed by \methodname{} in terms of AUROC and F1. This demonstrates that \methodname{} achieves a better tradeoff between utility and fairness than these GNN-based approaches.

In summary, \methodname{} on top of graph embedding approaches can compete or improve on fairness and utility with various specialized methods while outperforming them significantly in terms of efficiency.  % In other words compared with state-of-the-art fair representation learning techniques, \methodname{} achieves similar or better performance on utility and fairness, while significantly outperforming them in terms of efficiency.

%
% Ablation Study
%
\section{Full Ablation Study}
\label{sec:full_ablation_study}
\subsection{Tuning the Coarsen Level}
\label{sec:full_tune_cl}

%%%%
% Tuning the coarsen level
%%%%

\begin{figure*}[!t]
\centering
\includegraphics[width=.6\linewidth]{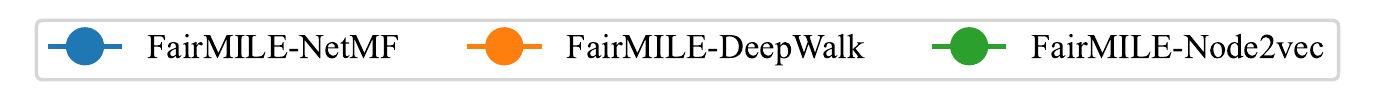}\\
\subfloat[German (AUROC)]{\includegraphics[width=0.180\linewidth]{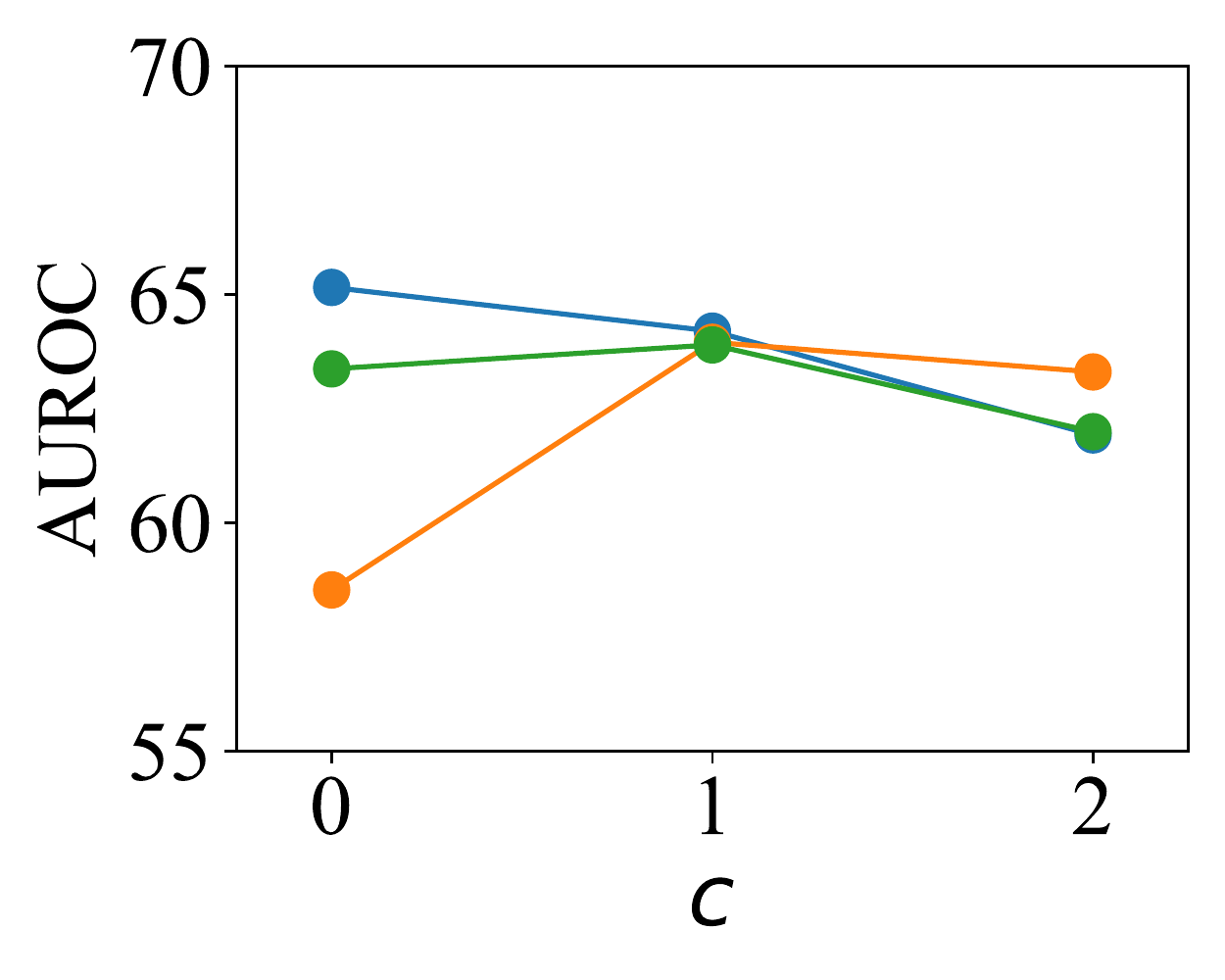}
\label{fig:german_auroc}}
\subfloat[German (F1)]{\includegraphics[width=0.180\linewidth]{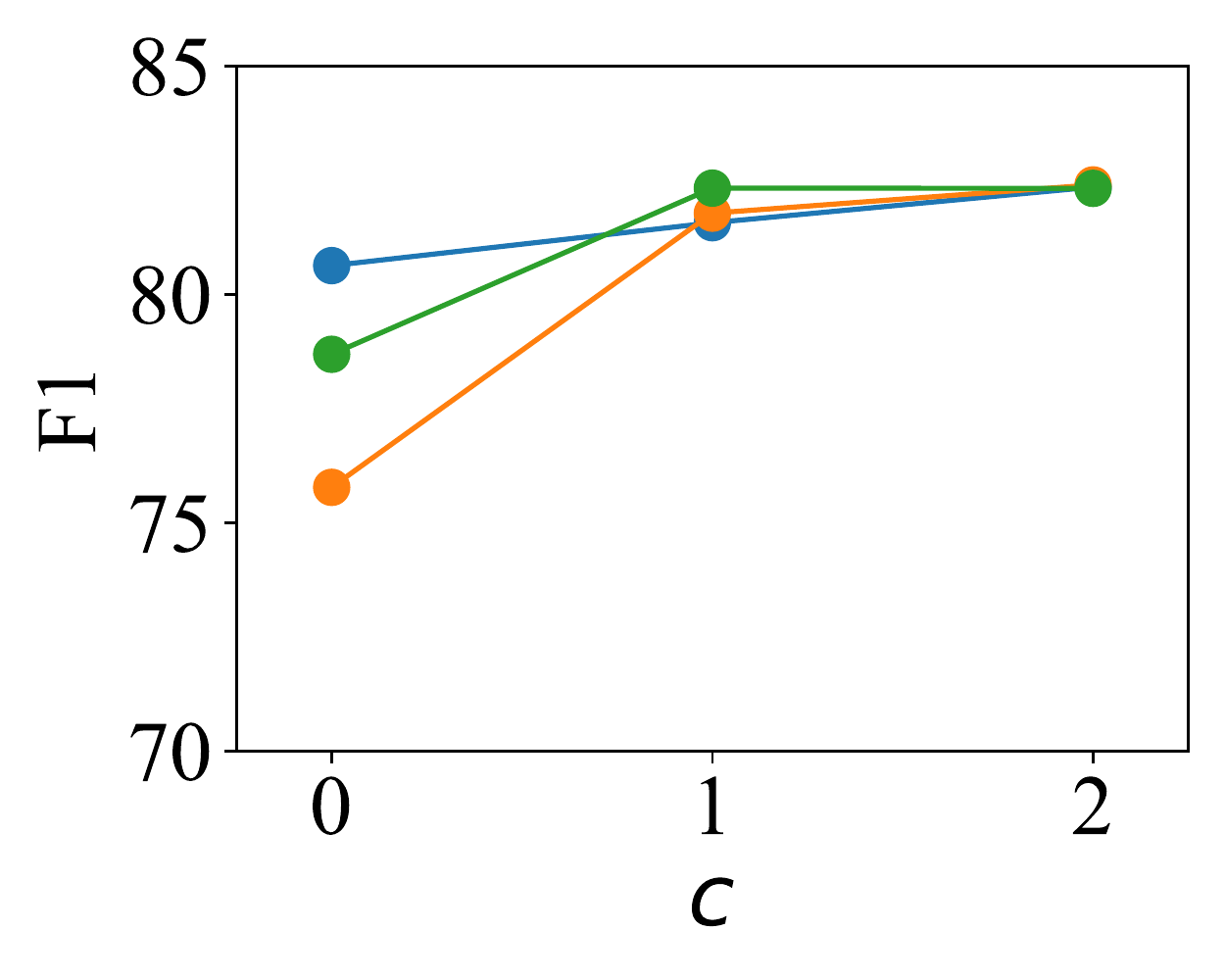}
\label{fig:german_f1}}
\subfloat[German ($\Delta_{DP}$)]{\includegraphics[width=0.180\linewidth]{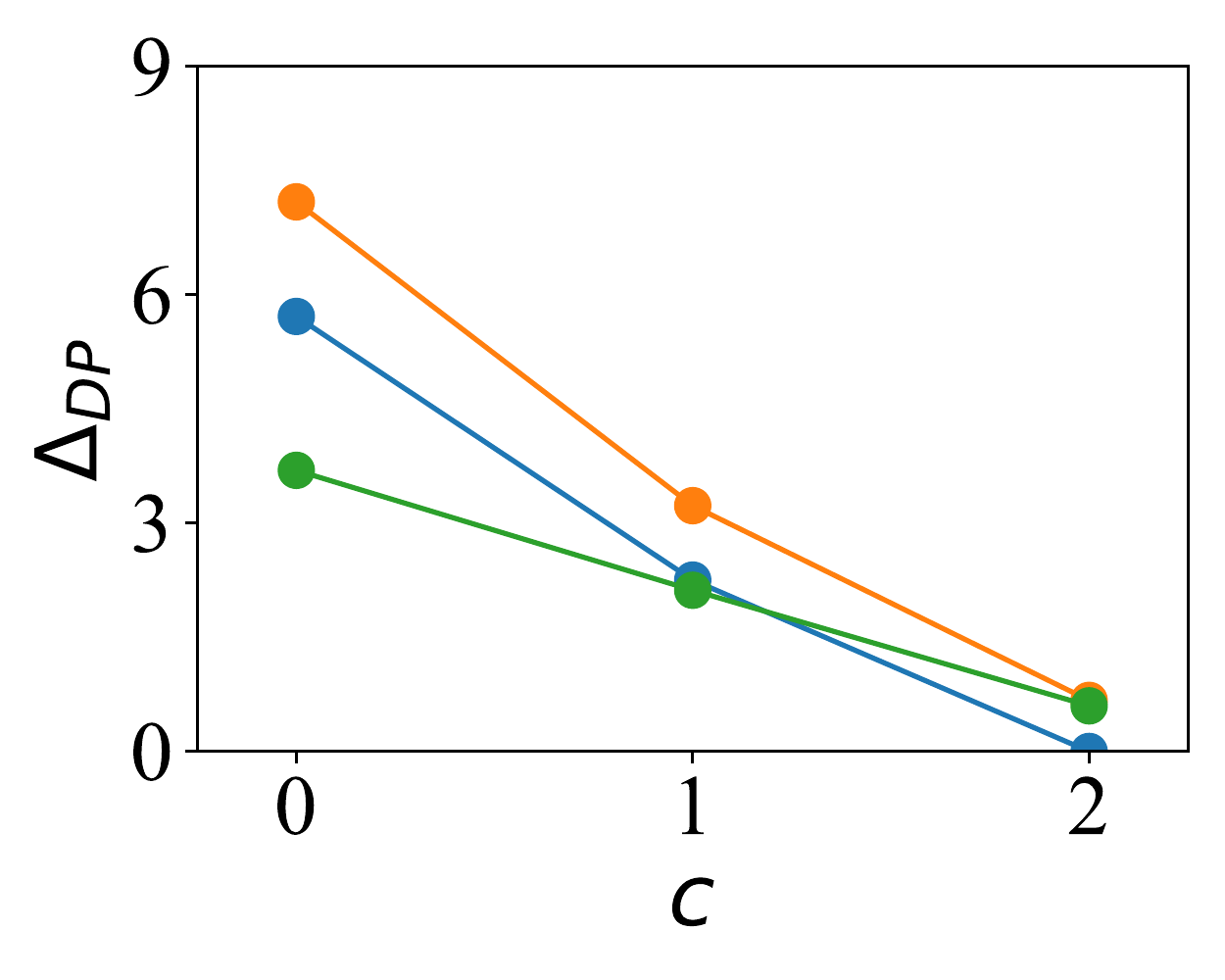}
\label{fig:german_dp}}
\subfloat[German ($\Delta_{EO}$)]{\includegraphics[width=0.180\linewidth]{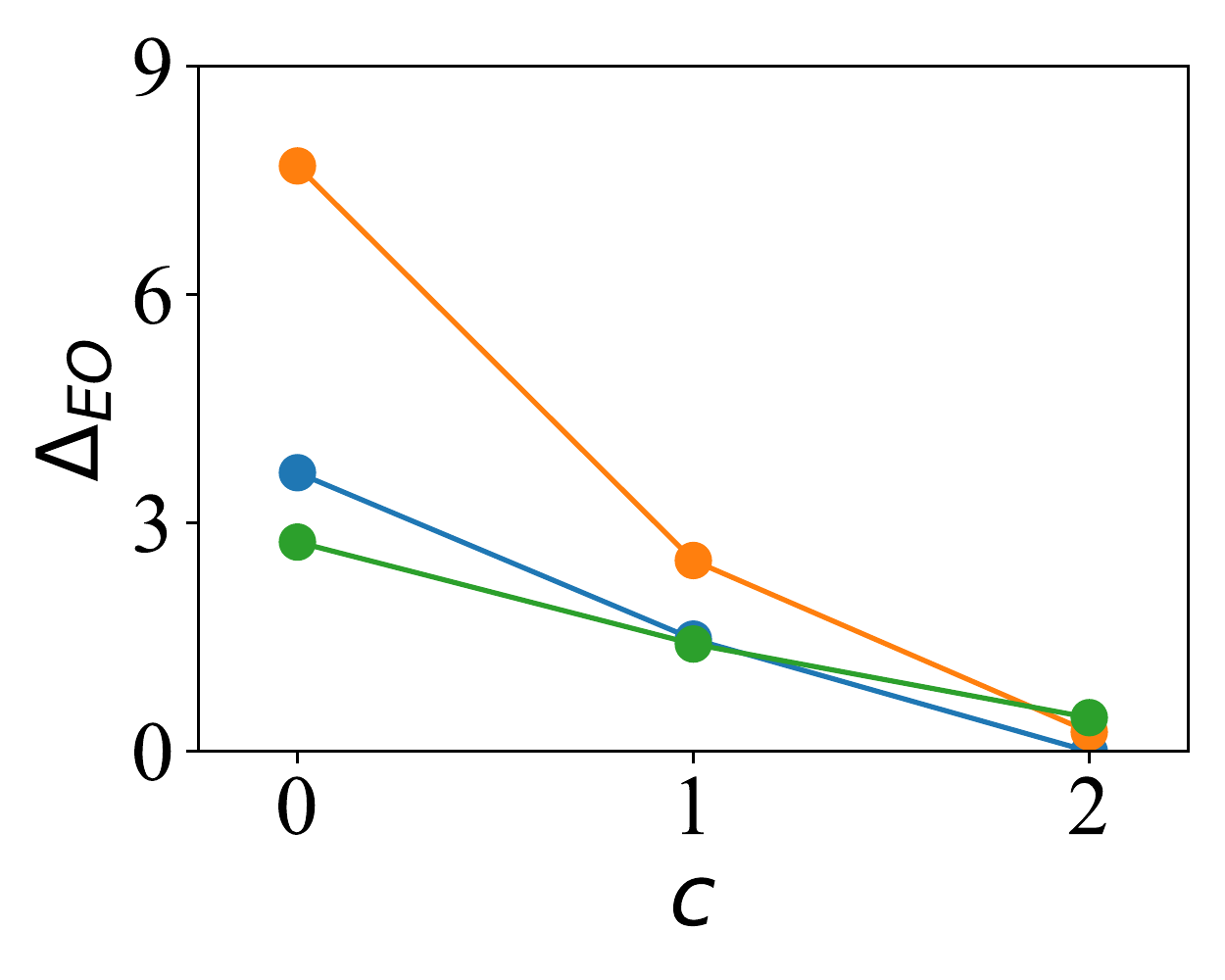}
\label{fig:german_eo}}
\subfloat[German (Time)]{\includegraphics[width=0.180\linewidth]{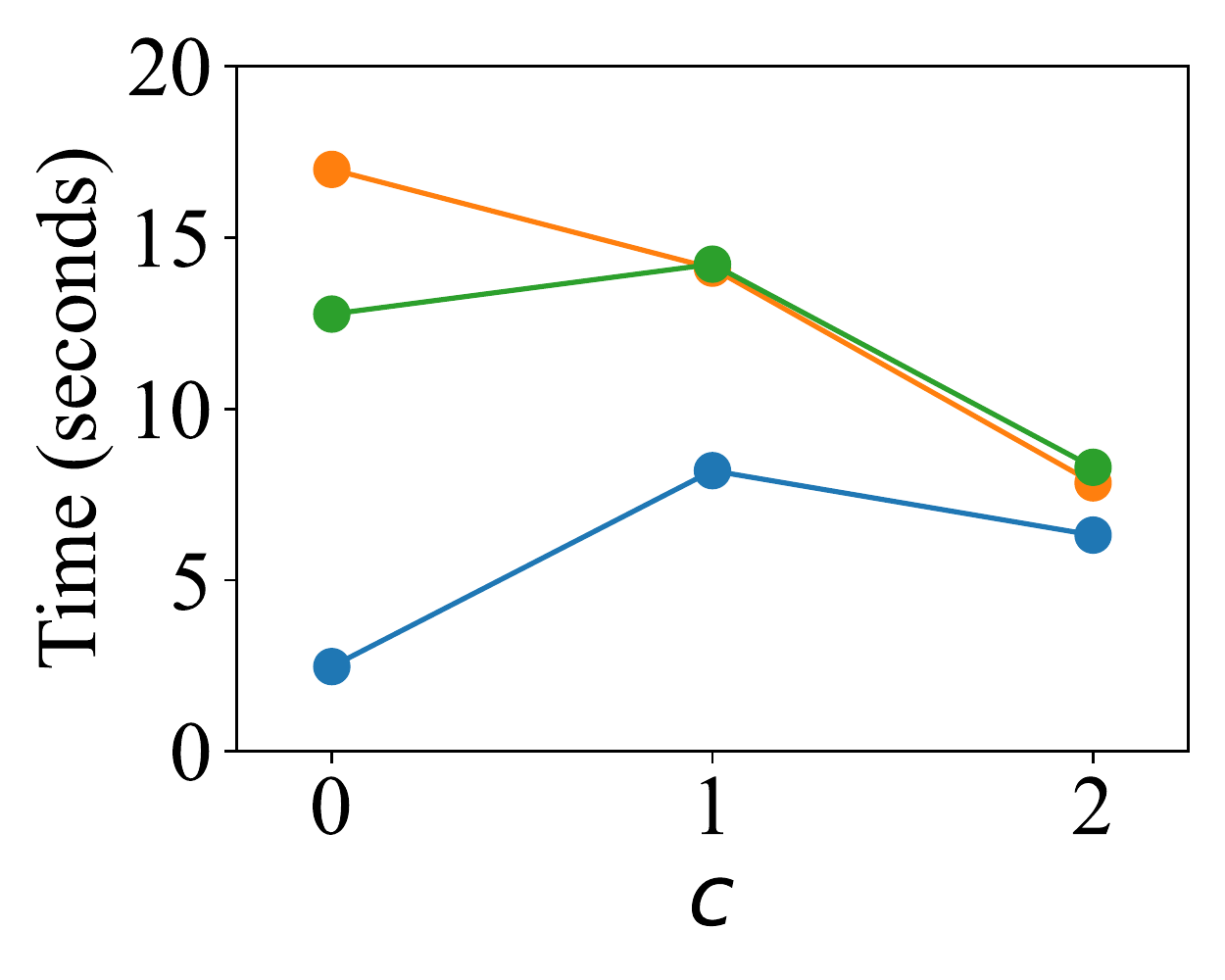}
\label{fig:german_time}} \\

\subfloat[Recidivism (AUROC)]{\includegraphics[width=0.180\linewidth]{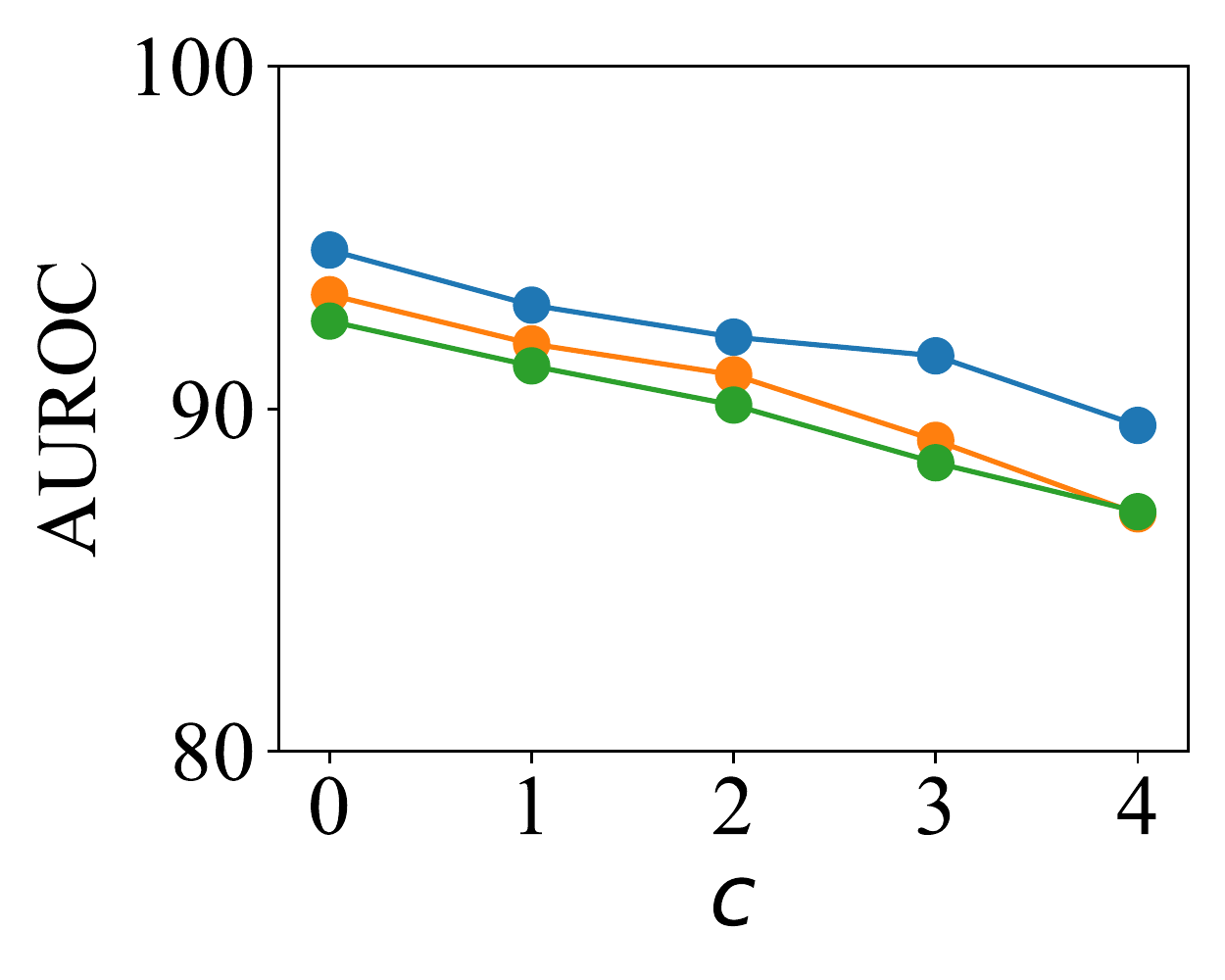}
\label{fig:bail_auroc}}
\subfloat[Recidivism (F1)]{\includegraphics[width=0.180\linewidth]{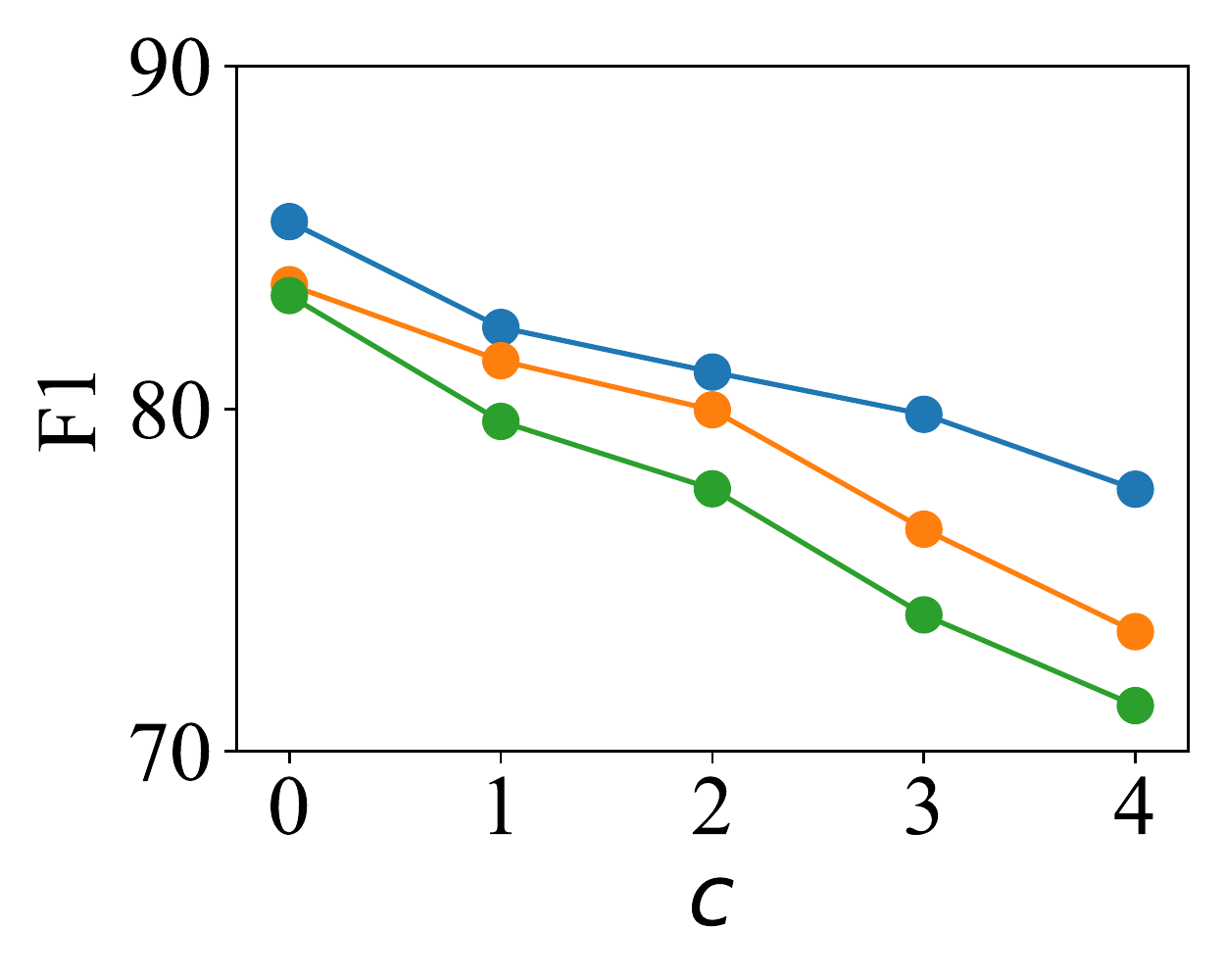}
\label{fig:bail_f1}}
\subfloat[Recidivism ($\Delta_{DP}$)]{\includegraphics[width=0.180\linewidth]{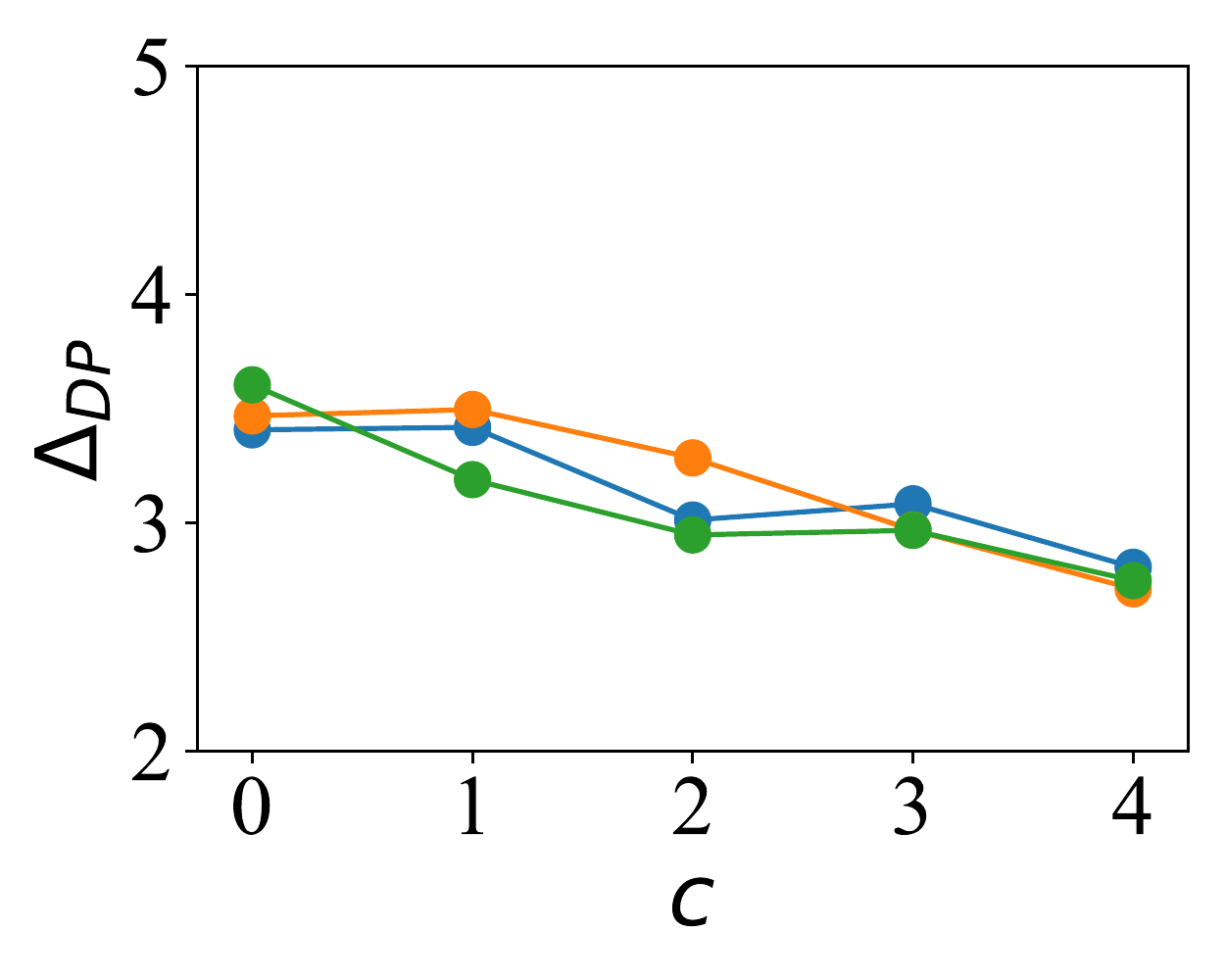}
\label{fig:bail_dp}}
\subfloat[Recidivism ($\Delta_{EO}$)]{\includegraphics[width=0.180\linewidth]{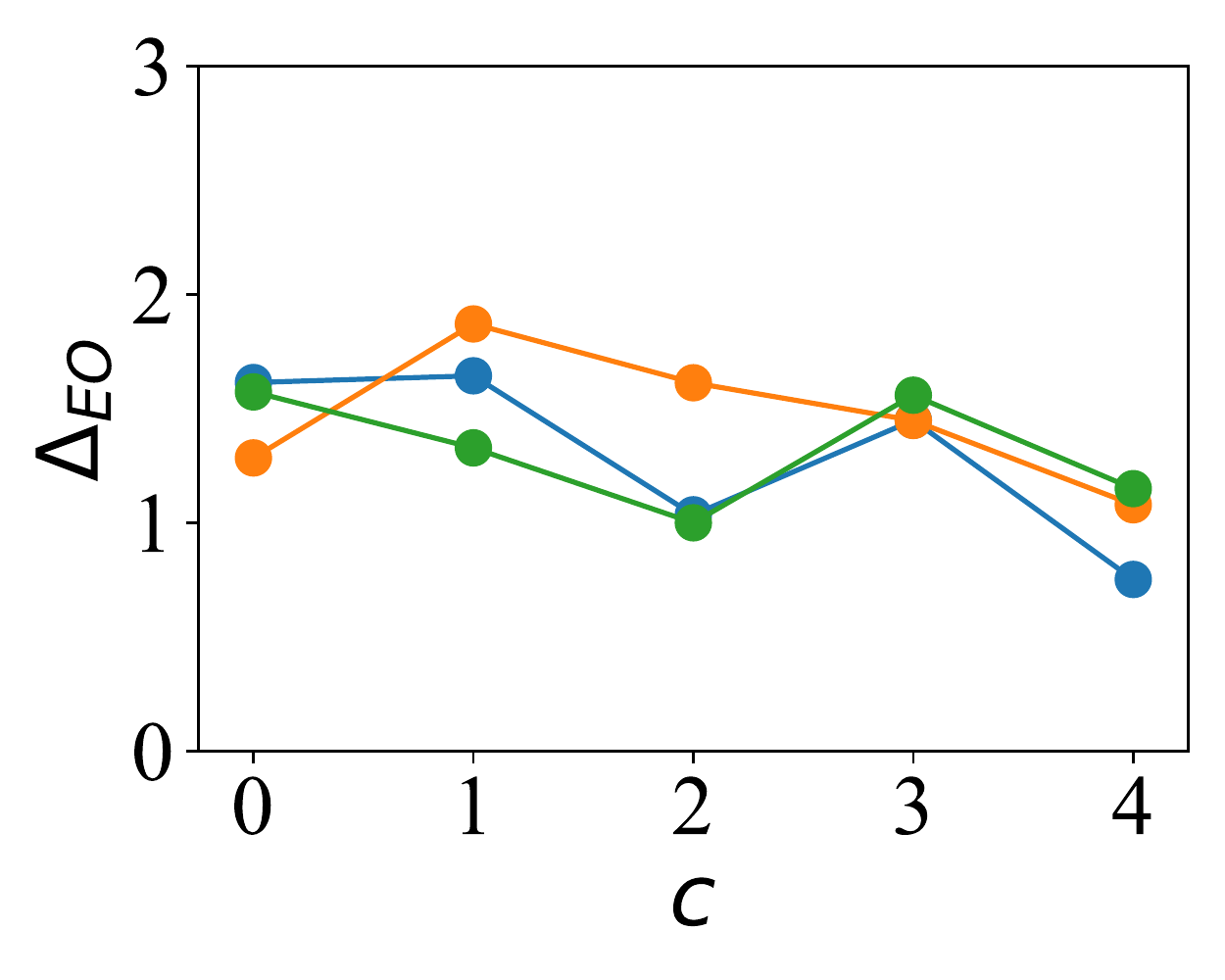}
\label{fig:bail_eo}}
\subfloat[Recidivism (Time)]{\includegraphics[width=0.180\linewidth]{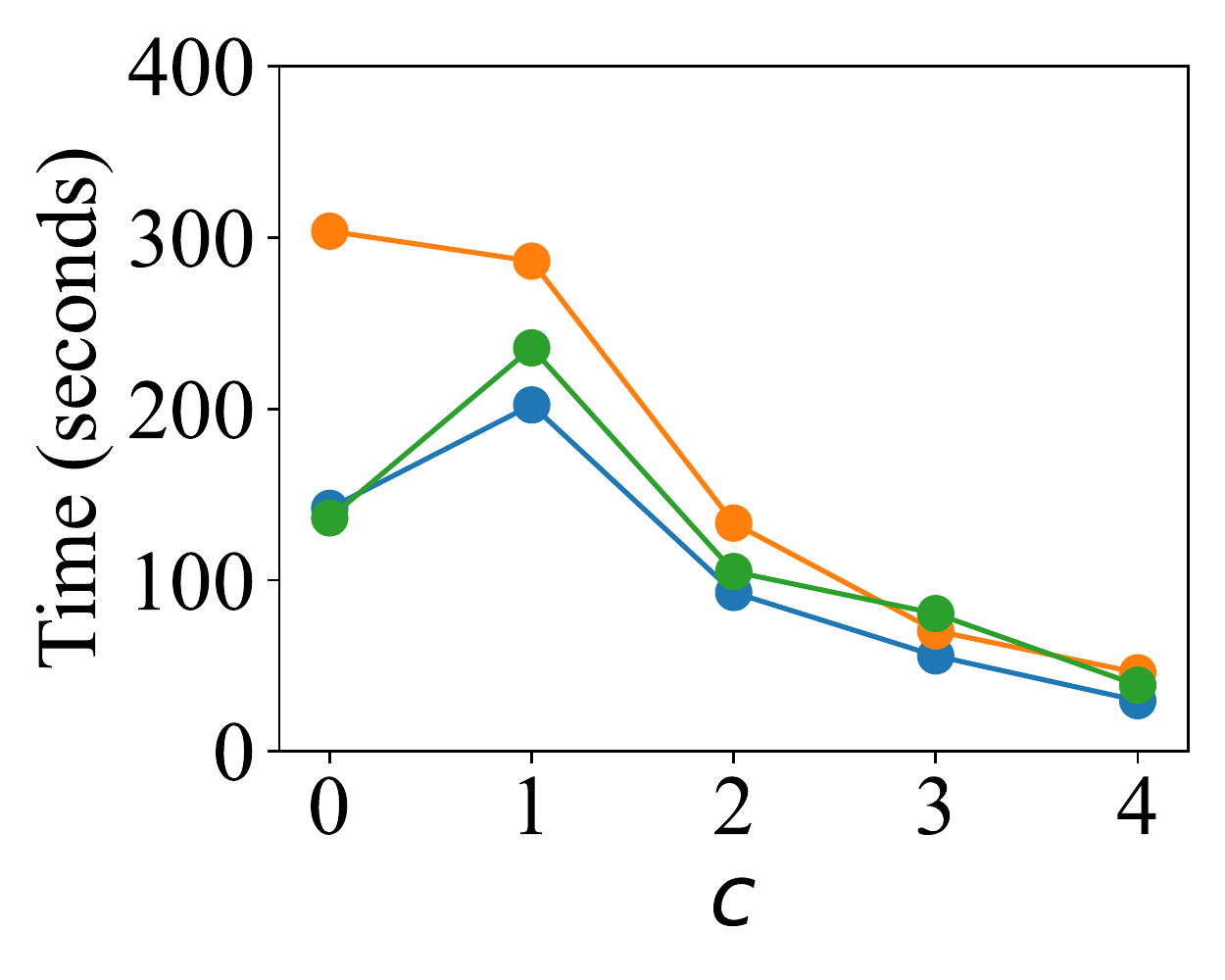}
\label{fig:bail_time}} \\

\subfloat[Credit (AUROC)]{\includegraphics[width=0.180\linewidth]{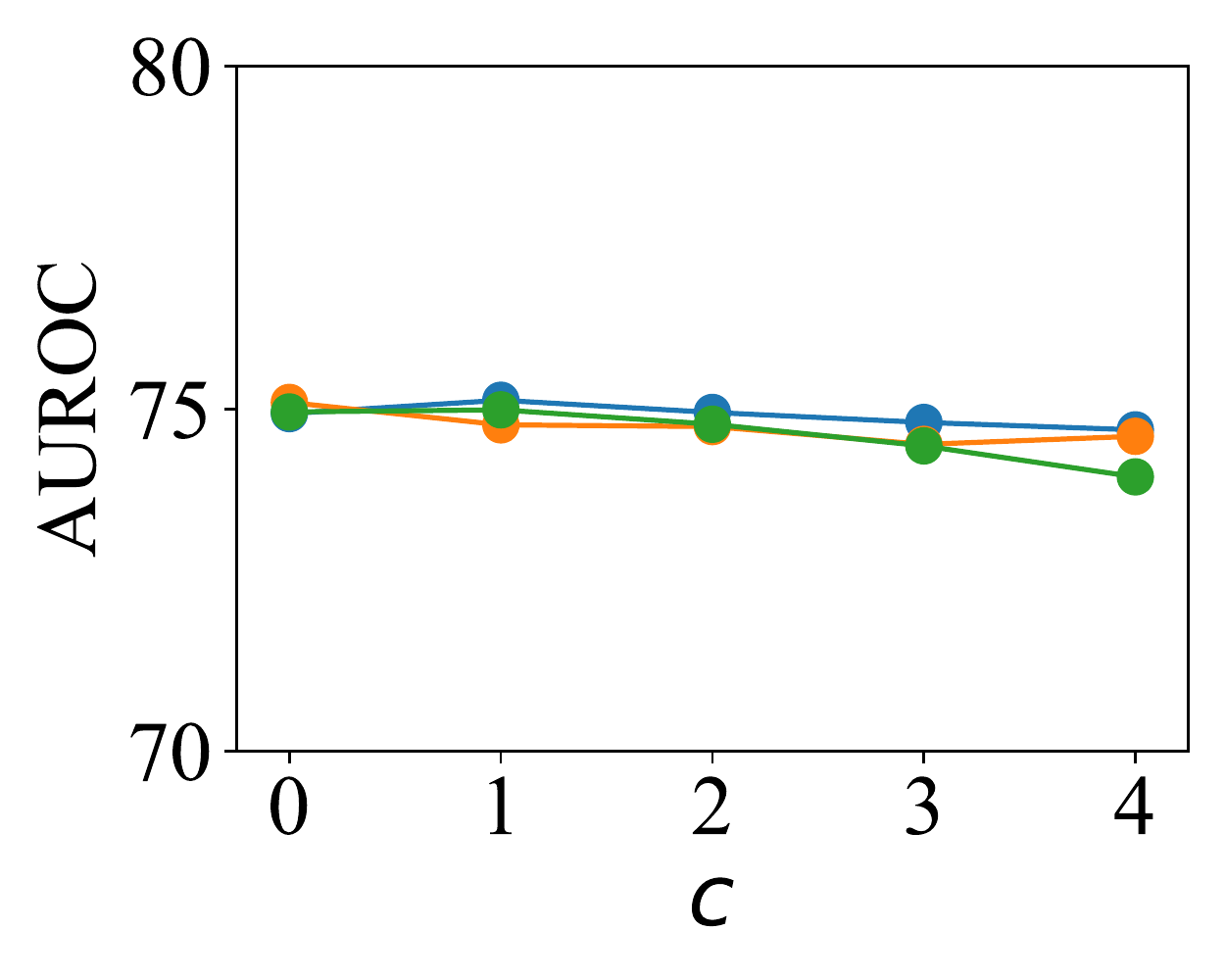}
\label{fig:credit_auroc}}
\subfloat[Credit (F1)]{\includegraphics[width=0.180\linewidth]{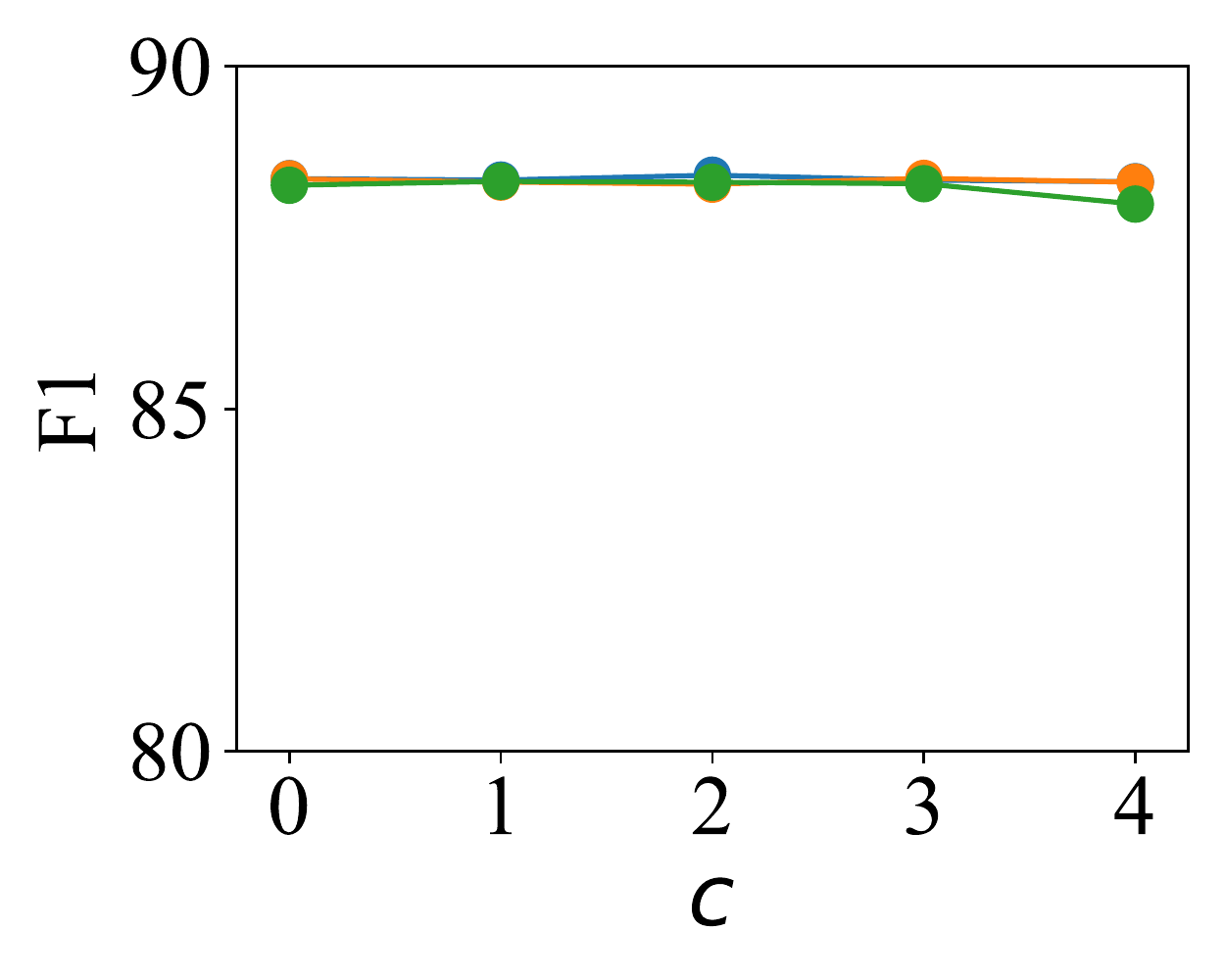}
\label{fig:credit_f1}}
\subfloat[Credit ($\Delta_{DP}$)]{\includegraphics[width=0.180\linewidth]{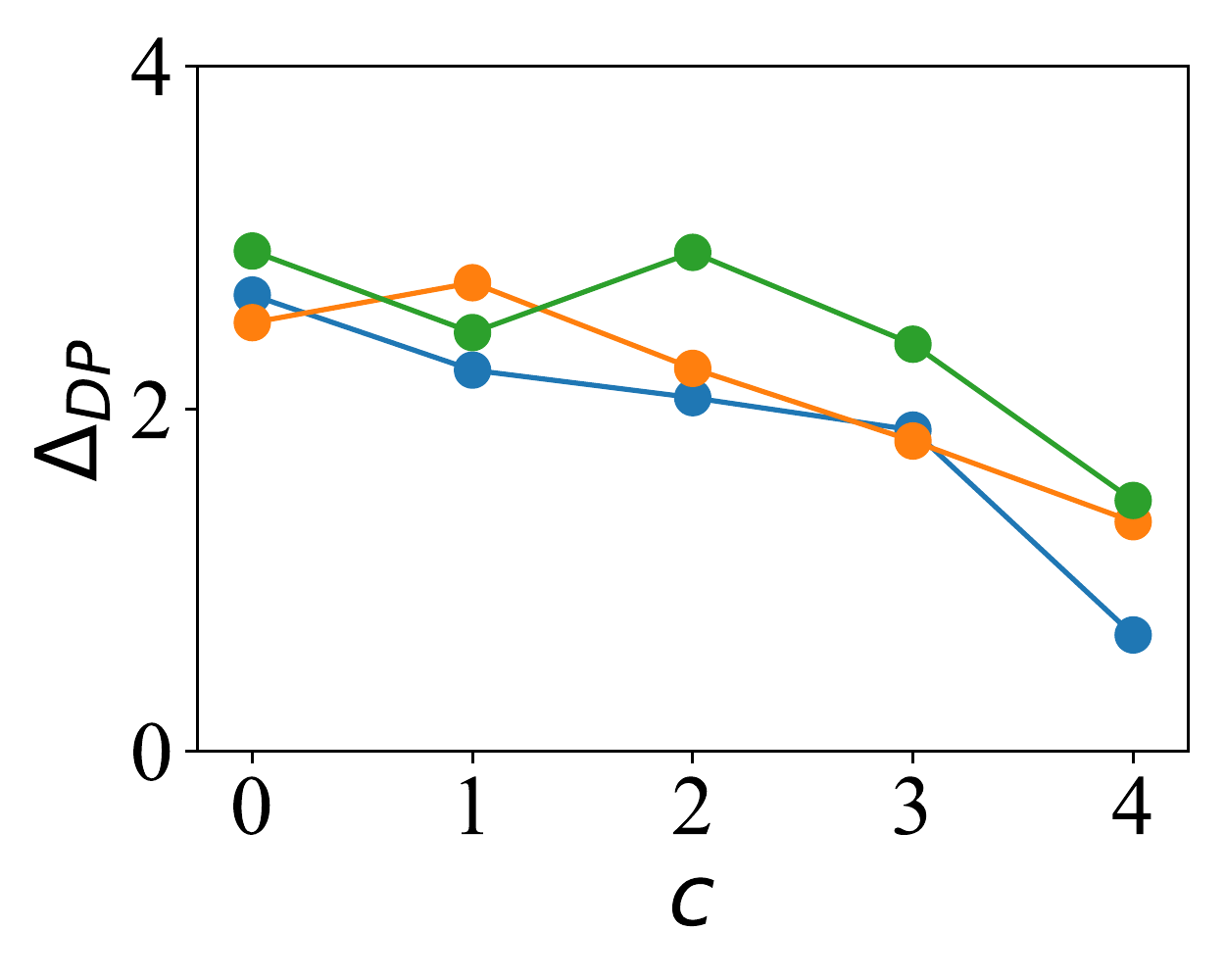}
\label{fig:credit_dp}}
\subfloat[Credit ($\Delta_{EO}$)]{\includegraphics[width=0.180\linewidth]{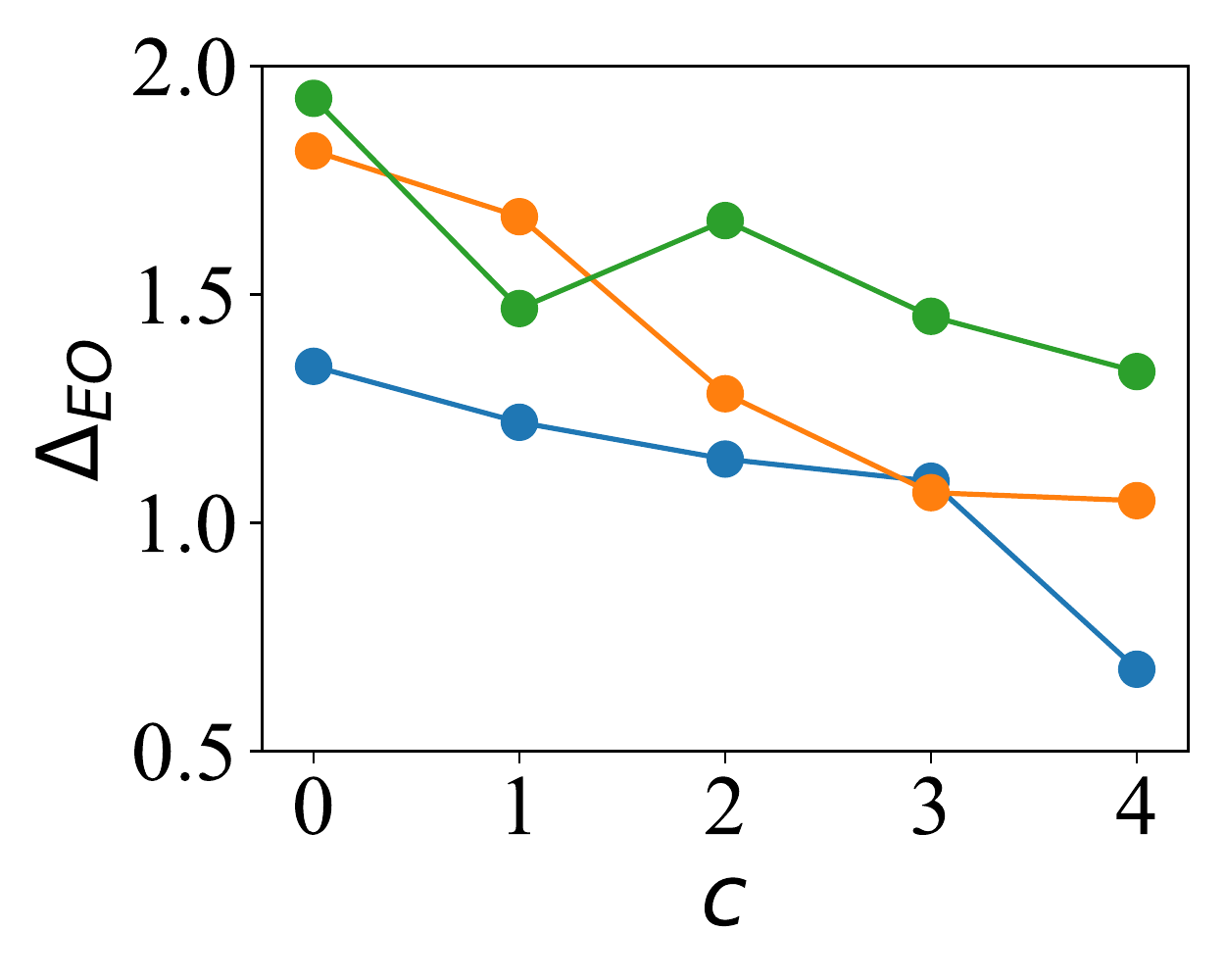}
\label{fig:credit_eo}}
\subfloat[Credit (Time)]{\includegraphics[width=0.180\linewidth]{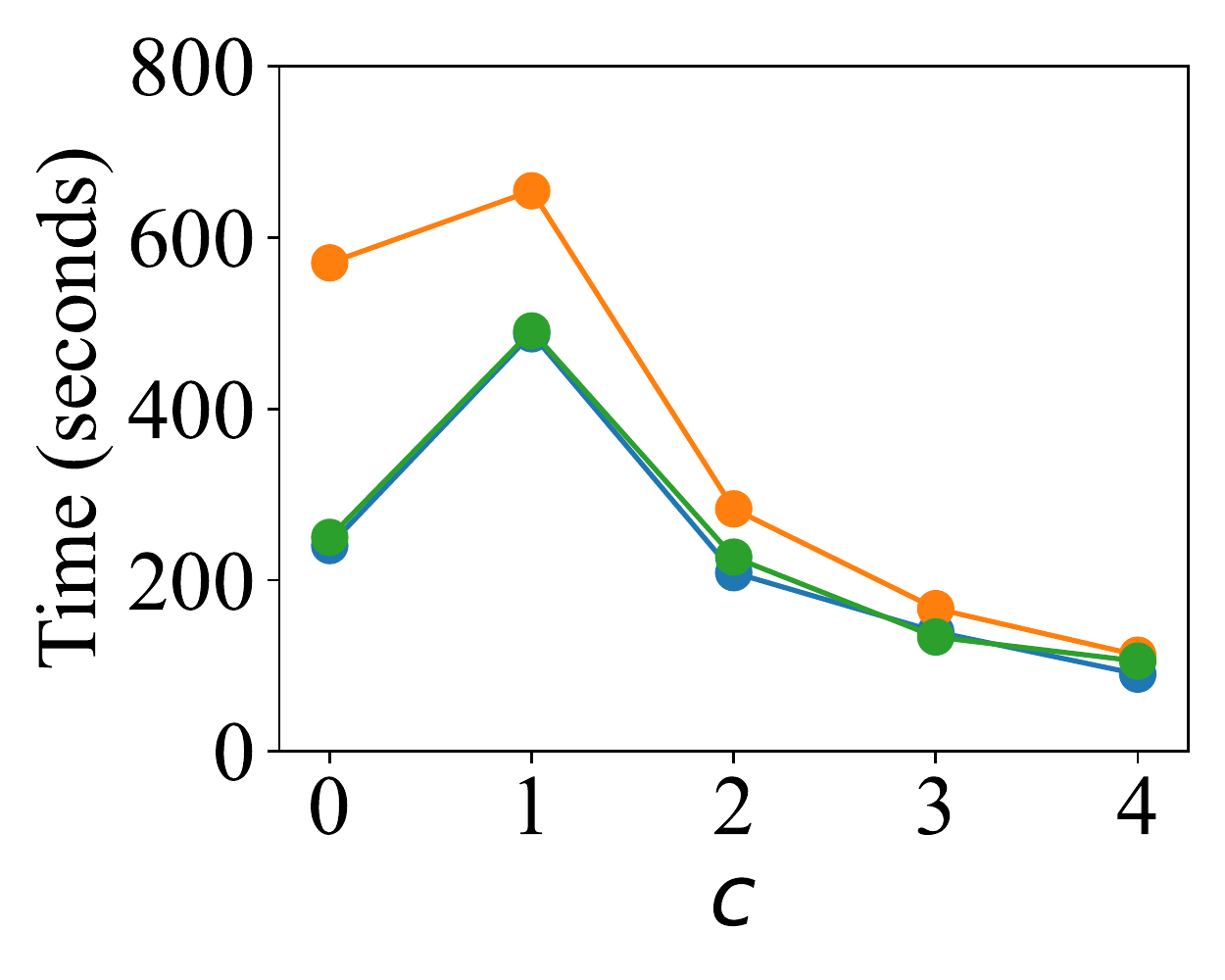}
\label{fig:credit_time}} 
\caption{Impact of coarsen level $c$ on \methodname{}'s utility, fairness, and efficiency.}
\label{fig:tune_cl}
\Description[Impact of coarsen level.]{The impact of coarsen level on utility, fairness, and efficiency.}
\vspace{-.1in}
\end{figure*}

We vary the coarsen level $c$ to observe its impact on utility, fairness, and efficiency. Results are shown in \autoref{fig:tune_cl}. Note that when $c=0$, \methodname{} is performing the base embedding method on the original graph. Generally, increasing $c$ leads to a slight decrease in AUROC and F1 scores. For example, the AUROC score of DeepWalk only decreases by $0.6\%$ after \methodname{} coarsens the graph 4 times. In some cases, \methodname{} achieves a better utility than the base embedding method (e.g., \methodname{}-Node2vec with $c=1$ on German). While the decrease of utility is negligible, increasing $c$ can visibly improve the fairness of representations. For example, vanilla DeepWalk has $\Delta_{DP}=7.22$ and $\Delta_{EO}=7.69$ on German, which is improved to $\Delta_{DP}=0.67$ and $\Delta_{EO}=0.26$ by \methodname{} ($c=2$). Last of all, increasing the coarsen level significantly improves the efficiency. Using a small $c$ may make \methodname{} slower because the time of coarsening and refinement outweighs the saved time of learning embedding when the coarsened graph is not small enough. Examples include $c=1$ on Credit. Given the little cost of utility, we suggest using a large $c$ for the sake of fairness and efficiency. % \autoref{table:vary_cl}

\begin{figure*}[!t]
\vskip -.1in
\centering
\subfloat[Varying $\lambda_r$]{\includegraphics[width=0.25\linewidth]{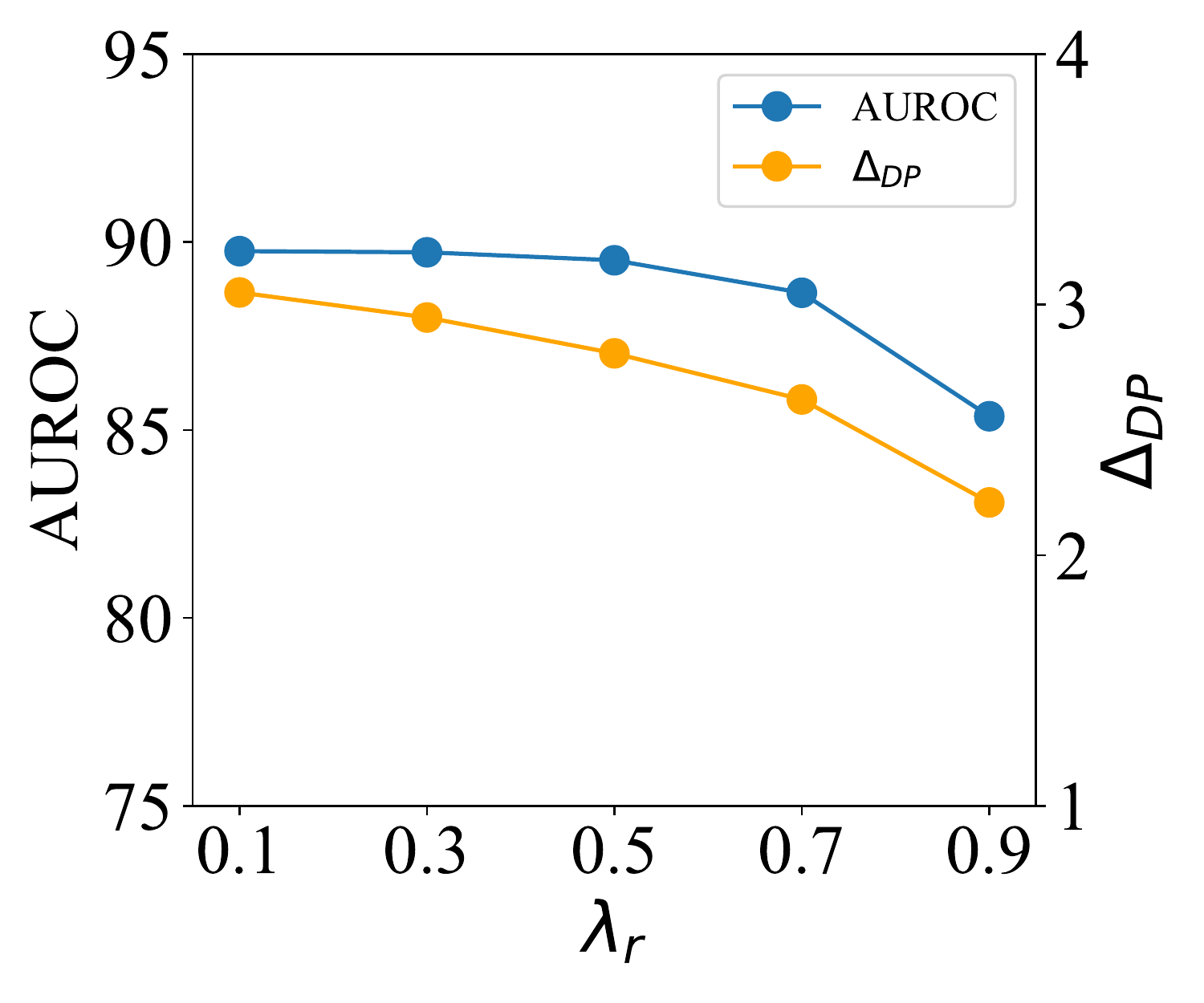}%
\label{fig:bail_fixlc}}
\subfloat[Varying $\lambda_c$]{\includegraphics[width=0.25\linewidth]{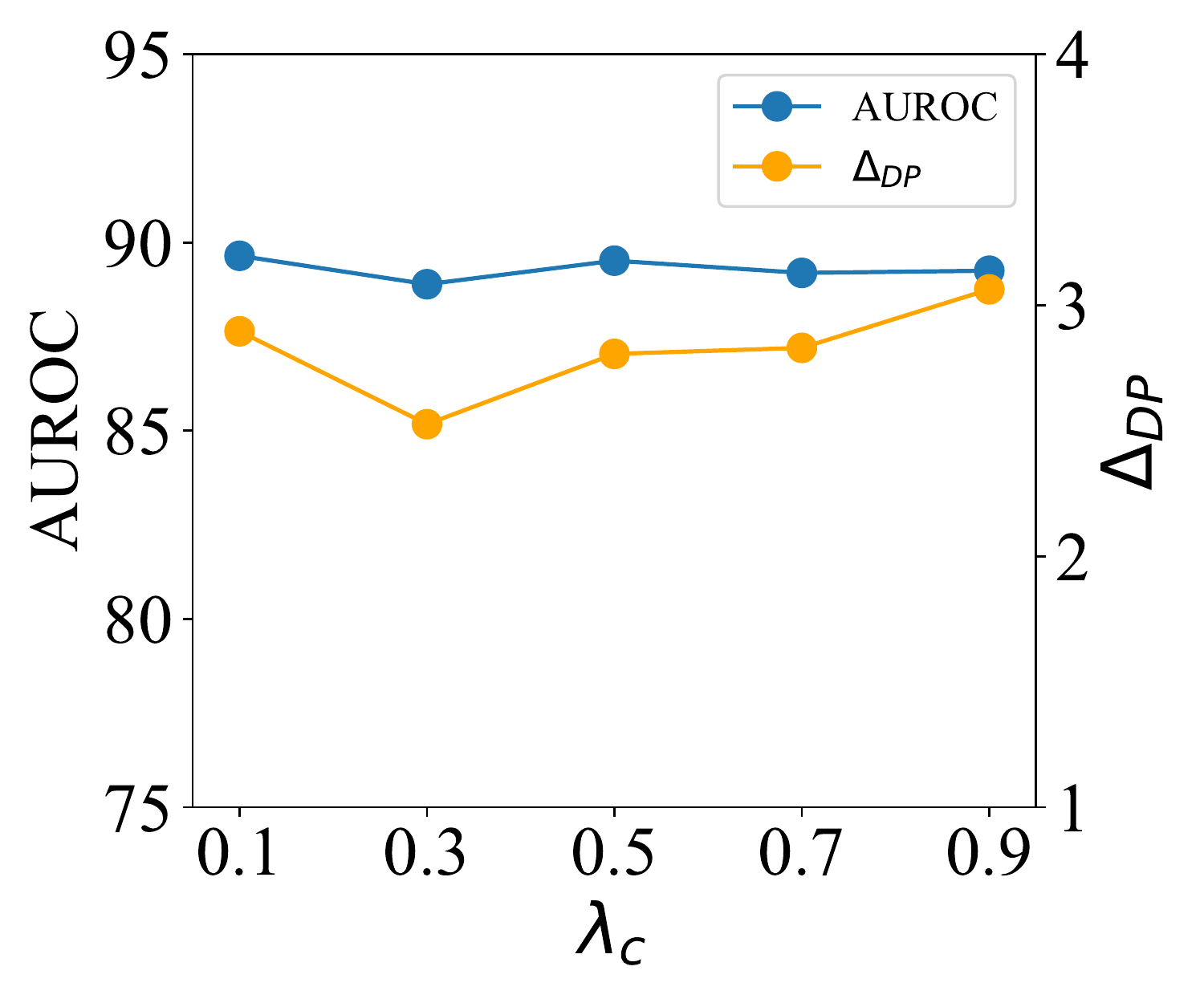}%
\label{fig:bail_fixll}}
\caption{Impact of varying $\lambda_c$ and $\lambda_r$ on utility and fairness on Recidivism dataset.}
\label{fig:lambda}
\Description[Impact of hyperparameters.]{Impact of varying lambdas on utility and fairness on Recidivism dataset.}
\vskip -.1in
\end{figure*}

\subsection{Trade-off between Utility and Fairness}
\label{sec:full_tradeoff}

To further explore the trade-off of \methodname{} between utility and fairness, we choose the values of $\lambda_c$ and $\lambda_r$ from $\{0.1, 0.3, 0.5, 0.7, 0.9\}$ respectively to observe the impact on performance. \autoref{fig:lambda} shows the results of \methodname{}-NetMF on Recidivism with $c=4$ (We only report these results for one dataset since results on other datasets are similar).  We use AUROC and $\Delta_{DP}$ as the metrics for utility and fairness.
%, and darker color refers to better results. 
It is clear that there is a trade-off between the utility scores and the fairness of learned embeddings on this dataset. 
Increasing fairness (represented by lower $\Delta_{DP}$) often causes a decrease in utility scores. We also observe that $\lambda_r$ has a larger impact on this tradeoff than $\lambda_c$. We also find in general that our choice of $\lambda_c$ = $\lambda_r = 0.5$ achieves a reasonable trade-off (applies to this dataset and the other datasets and tasks in our study). 
We do note of course that for different scenarios the designer may prefer to choose these parameters appropriately.

%%%%
% Comparison in Link Prediction
%%%%

\begin{table*}[!t]
\caption{Comparison in link prediction between \methodname{} and other baselines on Pubmed dataset.}
\label{table:link_prediction}
\vskip 0.1in
\begin{center}
\begin{small}
% \resizebox{.9\linewidth}{!}{
\begin{tabular}{c|l|lll|rr|r}
\toprule
Dataset & \makecell[c]{Method} & \makecell[c]{AUROC ($\uparrow$)} & \makecell[c]{AP ($\uparrow$)} & \makecell[c]{Accuracy ($\uparrow$)} & \makecell[c]{$\Delta_{DP, ~\mathrm{LP}}~(\downarrow)$} & \makecell[c]{$\Delta_{EO, ~\mathrm{LP}}~(\downarrow)$} & \makecell[c]{Time $(\downarrow)$}\\
\midrule
\multirow{11}{*}{Pubmed}
 & VGAE & \textbf{95.03 $\pm$ 0.18} & \textbf{94.96 $\pm$ 0.19} & \textbf{87.48 $\pm$ 0.21} & 39.37 $\pm$ 0.88 & 10.28 $\pm$ 1.58 & \textbf{347.80} \\
%  & EDITS & & & & & & \\
 & FairAdj$_{\mathrm{T2=2}}$ & 94.29 $\pm$ 0.17 & 94.07 $\pm$ 0.14 & 86.63 $\pm$ 0.27 & 37.12 $\pm$ 0.89 & 7.57 $\pm$ 1.56 & 2218.41 \\
 & FairAdj$_{\mathrm{T2=5}}$ & 93.57 $\pm$ 0.19 & 93.21 $\pm$ 0.13 & 85.69 $\pm$ 0.16 & 35.06 $\pm$ 1.01 & 5.60 $\pm$ 1.55 & 2480.66 \\
 & FairAdj$_{\mathrm{T2=20}}$ & 91.78 $\pm$ 0.12 & 91.27 $\pm$ 0.24 & 83.20 $\pm$ 0.22 & \textbf{30.41 $\pm$ 0.89} & \textbf{2.41 $\pm$ 1.28} & 4532.99 \\
 & CFGE & 91.25 $\pm$ 5.32 & 92.08 $\pm$ 5.22 & 82.78 $\pm$ 5.79 & 33.03 $\pm$ 6.19 & 9.55 $\pm$ 2.35 & 4237.43 \\
 \cmidrule{2-8}
 & NetMF & \textbf{98.43 $\pm$ 0.07} & \textbf{98.26 $\pm$ 0.05} & 93.86 $\pm$ 0.19 & 38.59 $\pm$ 0.14 & \textbf{2.04 $\pm$ 0.15} & 281.35 \\
 & \methodname{}-NetMF & 98.11 $\pm$ 0.12 & 97.29 $\pm$ 0.20 & \textbf{94.84 $\pm$ 0.31} & \textbf{31.97 $\pm$ 0.68} & 2.70 $\pm$ 0.22 & \textbf{126.17} \\
 \cmidrule{2-8}
 & DeepWalk & 98.35 $\pm$ 0.14 & 98.05 $\pm$ 0.17 & 91.77 $\pm$ 0.29 & 35.02 $\pm$ 0.43 & 0.40 $\pm$ 0.12 & 354.27 \\\
 & \methodname{}-DeepWalk & \textbf{99.57 $\pm$ 0.04} & \textbf{99.32 $\pm$ 0.08} & \textbf{97.61 $\pm$ 0.06} & \textbf{27.30 $\pm$ 0.23} & \textbf{0.37 $\pm$ 0.11} & \textbf{201.03} \\
 \cmidrule{2-8}
 & Node2vec & \textbf{99.52 $\pm$ 0.04} & \textbf{99.44 $\pm$ 0.04} & 93.11 $\pm$ 0.21 & 40.28 $\pm$ 0.41 & \textbf{0.21 $\pm$ 0.13} & 249.52 \\
 & Fairwalk & 99.50 $\pm$ 0.05 & 99.43 $\pm$ 0.05 & 92.86 $\pm$ 0.24 & 38.58 $\pm$ 0.35 & 0.65 $\pm$ 0.12 & 225.99 \\
 & \methodname{}-Node2vec & 99.23 $\pm$ 0.07 & 98.68 $\pm$ 0.14 & \textbf{96.43 $\pm$ 0.06} & \textbf{26.51 $\pm$ 0.35} & 0.59 $\pm$ 0.05 & \textbf{143.01} \\
\bottomrule
\end{tabular}
% }
\end{small}
\end{center}
\vskip -0.1in
\end{table*}

\section{Full Results for Link Predictions}
\label{sec:full_link_prediction}

We evaluate \methodname{} in the context of link prediction on three datasets. For \methodname{}, we set $c=2$ on smaller datasets (Cora and Citeseer) and $c=4$ on Pubmed.  \autoref{table:link_prediction} shows the results on Pubmed. For the results on other datasets, please refer to \autoref{table:link_prediction_short} in Section \ref{sec:link_prediction_short}. First, \methodname{} makes fair predictions on all datasets. Our framework has an improvement of up to $45.3\%$ on $\Delta_{DP, ~\mathrm{LP}}$ compared with the base embedding approaches. In terms of $\Delta_{EO, ~\mathrm{LP}}$, while the performance of \methodname{} declines on Pubmed very slightly ($2.70\%$ v.s. $2.04\%$ in NetMF), it greatly reduces the unfair predictions on Cora and Citeseer. Combining the observations on both metrics, \methodname{} successfully enforces fairness in the task of link prediction. When compared with FairWalk, \methodname{}-Node2vec always has a better fairness score (e.g., $12.49\%$ v.s. $23.59\%$ on Citeseer). In addition, we notice that FairAdj is less biased than VGAE, which demonstrates its effectiveness in debiasing. However, its best performance with $T2=20$ is still outperformed by \methodname{} on all datasets. For example, the $\Delta_{DP,~\mathrm{LP}}$ score of \methodname{}-Node2vec on Citeseer is $45.1\%$ lower than that of FairAdj ($T2=20$). Compared with CFGE, \methodname{} on top of Node2vec has a better performance in terms of fairness.

On the other hand, \methodname{} also performs well in terms of utility. In comparison to the standard embedding approaches, \methodname{} achieves a similar or better utility performance. For example, \methodname{} {\bf actually enhances the accuracy} of DeepWalk from $91.77\%$ to $97.61\%$ on Pubmed. Similar results can also be observed on the other metrics and datasets. Compared with VGAE-based methods, \methodname{} outperforms them again on utility. Examples include that AUROC scores of VGAE and \methodname{}-DeepWalk on Pubmed are $95.03\%$ v.s. $99.57\%$, respectively.

Finally, \methodname{} is more efficient than other baselines. For example, on the largest dataset Pubmed, \methodname{}-NetMF takes around 2 minutes, while NetMF needs around 5 minutes, and FairAdj with $T2=20$ even requires more than one hour to finish. In summary, \methodname{} can flexibly generalize to the link prediction task improving over the state of the art on both counts of fairness and efficiency at a marginal cost to utility. 